\documentclass[english,a4paper,12pt]{article}

\usepackage{amsxtra, amsfonts, amssymb, amstext, amsmath, mathtools}%, 
\usepackage{amsthm}
\usepackage{booktabs}
\usepackage{nicefrac}
\usepackage{xspace}
\usepackage{url}
\usepackage{graphics,color}
\usepackage[algo2e,ruled,vlined,linesnumbered]{algorithm2e}
\usepackage{wrapfig}

\newtheorem{theorem}{Theorem}
\newtheorem{lemma}[theorem]{Lemma}

\newcommand{\oea}{\mbox{$(1 + 1)$~EA}\xspace}

\newcommand{\mplea}{\mbox{$(\mu+\lambda)$~EA}\xspace}
\newcommand{\mclea}{\mbox{$(\mu,\lambda)$~EA}\xspace}

\newcommand{\ollga}{\mbox{$(1+(\lambda,\lambda))$~GA}\xspace}

\newcommand{\OM}{\textsc{OneMax}\xspace}
\newcommand{\onemax}{\OM}
\newcommand{\LO}{\textsc{Leading\-Ones}\xspace}
\newcommand{\leadingones}{\LO}

\newcommand{\binval}{\textsc{BinVal}\xspace}
\DeclareMathOperator{\jump}{\textsc{Jump}}

\DeclareMathOperator{\Sample}{Sample}
\DeclareMathOperator{\minmax}{minmax}
\DeclareMathOperator{\paral}{par}
\DeclareMathOperator{\poly}{poly}
\newcommand{\Ymax}{Y_{\max}}

\newcommand{\R}{\ensuremath{\mathbb{R}}}

\newcommand{\N}{\ensuremath{\mathbb{N}}} % ohne Null!!!

\newcommand{\calA}{\ensuremath{\mathcal{A}}} 
\newcommand{\calF}{\ensuremath{\mathcal{F}}} 
\newcommand{\calP}{\ensuremath{\mathcal{P}}}

\DeclareMathOperator{\Bin}{Bin}

\newcommand{\eps}{\varepsilon}

\newcommand{\assign}{\leftarrow}

\begin{document}

\title{The Runtime of the Compact Genetic Algorithm on Jump Functions\thanks{Extended version of results that appeared at GECCO 2019~\cite{Doerr19gecco} and FOGA 2019~\cite{Doerr19foga}. It contains as new result the $\Omega(\mu \sqrt n + n \log n)$ lower bound. All other results have been significantly rewritten, both to polish the arguments and to give a more unified treatment of the two previous works. In this process, the GECCO 2019 results were extended to subjump functions, the FOGA 2019 results were extended to superjump functions -- two natural extensions of the jump functions class. This work was supported by a public grant as part of the Investissement d'avenir project, reference ANR-11-LABX-0056-LMH, LabEx LMH, in a joint call with Gaspard Monge Program for optimization, operations research and their interactions with data sciences.
}}

%\author{Benjamin Doerr\\ \'Ecole Polytechnique\\ CNRS\\ Laboratoire d'Informatique (LIX)\\ Palaiseau\\ France}

\author{Benjamin Doerr\\ Laboratoire d'Informatique (LIX)\\ CNRS\\ \'Ecole Polytechnique\\ Institut Polytechnique de Paris\\ Palaiseau\\ France%\\ email: {\tt doerr@lix.polytechnique.fr}
}

\maketitle

\sloppy{
\begin{abstract}
  In the first and so far only mathematical runtime analysis of an estimation-of-distribution algorithm (EDA) on a multimodal problem, Hasen\"ohrl and Sutton (GECCO 2018) showed for any $k = o(n)$ that the compact genetic algorithm (cGA) with any hypothetical population size $\mu = \Omega(ne^{4k} + n^{3.5+\eps})$ with high probability finds the optimum of the $n$-dimensional jump function with jump size $k$ in time $O(\mu n^{1.5} \log n)$.
  
  We significantly improve this result for small jump sizes $k \le \frac 1 {20} \ln n -1$. In this case, already for $\mu = \Omega(\sqrt n \log n) \cap \poly(n)$ the runtime of the cGA with high probability is only $O(\mu \sqrt n)$. For the smallest admissible values of $\mu$, our result gives a runtime of $O(n \log n)$, whereas the previous one only shows $O(n^{5+\eps})$. Since it is known that the cGA with high probability needs at least $\Omega(\mu \sqrt n)$ iterations to optimize the unimodal $\onemax$ function, our result shows that the cGA in contrast to most classic evolutionary algorithms here is able to cross moderate-sized valleys of low fitness at no extra cost. 
    
  For large $k$, we show that the exponential (in $k$) runtime guarantee of Hasen\"ohrl and Sutton is tight and cannot be improved, also not by using a smaller hypothetical population size. We prove that any choice of the hypothetical population size leads to a runtime that, with high probability, is at least exponential in the jump size $k$. This result might be the first non-trivial exponential lower bound for EDAs that holds for arbitrary parameter settings. 
  
  To complete the picture, we show that the cGA with hypothetical population size $\mu = \Omega(\log n)$ with high probability needs $\Omega(\mu \sqrt n + n \log n)$ iterations to optimize any $n$-dimensional jump function. This bound was known for \onemax, but, as we also show, the usual domination arguments do not allow to extend lower bounds on the performance of the cGA on \onemax to arbitrary functions with unique optimum.
  
  As a side result, we provide a simple general method based on parallel runs that, under mild conditions, (i)~overcomes the need to specify a suitable population size and still gives a performance close to the one stemming from the best-possible population size, and (ii)~transforms EDAs with high-probability performance guarantees into EDAs with similar bounds on the expected runtime.
\end{abstract}

%\begin{large}
\section{Introduction}

Estimation-of-distribution algorithms (EDAs)~\cite{LarranagaL02,PelikanHL15} are a particular class of evolutionary algorithms. Whereas typical classic evolutionary algorithms evolve a population of (hopefully good) solutions, EDAs evolve a probabilistic model of the search space, that is, a probability distribution over the set of all solutions. The target is to obtain distributions that allow to easily sample good solutions for the optimization problem regarded. 

While the mathematical analysis of classical evolutionary algorithms (EAs) has produced a plethora of insightful results, see, e.g.,~\cite{NeumannW10,AugerD11,Jansen13,DoerrN20}, the rigorous understanding of EDAs is much less developed, see, e.g., the recent survey~\cite{KrejcaW20bookchapter}. Obviously, this is due to the highly complex stochastic processes that describe the runs of such algorithms. In consequence, despite significant efforts and deep results~\cite{Droste06,SudholtW19,LenglerSW18}, not even the runtime of the compact genetic algorithm (cGA) on the \onemax benchmark function is fully understood (here we would argue that the cGA is the simplest EDA and that the unimodal \onemax function, counting the number of ones in a bit string, is the easiest optimization problem with unique global optimum). It is therefore not surprising that many questions which are well-understood for EAs are only started to be understood for EDAs. 

One such question is how EDAs optimize objective functions that are not unimodal. In the first and, prior to this work, only runtime analysis of an EDA on a multimodal problem, Hasen\"ohrl and Sutton~\cite{HasenohrlS18} regard the optimization time of the cGA on the jump function class. These functions are unimodal apart from having a valley of low fitness of scalable size $k$ around the global optimum. For a sufficiently large constant $C$ and any constant $\eps > 0$, they show~\cite[Theorem~3.3]{HasenohrlS18} that the cGA with hypothetical population size $\mu \ge \max\{C n e^{4k}, n^{3.5+\eps}\}$\footnote{In the paper, this is stated as minimum of the two terms, but from the proofs it is clear that it should be the maximum.} with probability $1 - o(1)$ finds the optimum of any jump function with jump size $k = o(n)$ in $O(\mu n^{1.5} \log n)$ generations (which is also the number of fitness evaluations, since the cGA evaluates only two search points in each iteration).

This result is remarkable in that it shows that the cGA with the right choice of $\mu$ and for $k \ge 6$ is more efficient on jump functions than most evolutionary algorithms, which have a runtime of at least $\Omega(n^k)$; see Section~\ref{ssec:prevjump}.

\subsection{An Improved Upper Bound for Small Jump Sizes}

When the jump size $k$ is small, the runtime guarantee given by Hasen\"ohrl and Sutton~\cite{HasenohrlS18} is still relatively large. We note that even when choosing the smallest possible population size $\mu = n^{3.5+\eps}$, the runtime guarantee becomes at least $\Omega(n^{5+\eps})$. While clearly a polynomial runtime, and thus \emph{efficient} in the classic complexity theory view, this is a runtime that is not practical in many applications. Also, this runtime guarantee is weaker than the $O(n^k)$ bound for simple mutation-based EAs such as the \oea when $k \le 5$. Hence one could feel that the result of Hasen\"ohrl and Sutton shows the superiority of EDAs rather for problem instances for which both the runtime of typical EAs and the performance guarantee for the cGA are prohibitively large. In a similar vein, one has to question if a practitioner would run the cGA with a hypothetical population size of more than $n^{3.5}$ when solving a problem defined over bit strings of length $n$.

\textbf{Our first main result} is that these potential weaknesses of the cGA are not real and that the cGA performs in fact much better than what the previous work shows. We prove rigorously that the cGA with hypothetical population size $\mu \ge K \sqrt n \log n$, $K$ a sufficiently large constant, and $\mu$ polynomially bounded in $n$, with high probability\footnote{that is, with probability $1 - o(1)$, where the asymptotics is in $n$ for a fixed $k$ (which might be a function of $n$)} optimizes any $n$-dimensional jump function with jump size $k \le \frac 1 {20} \ln n -1$ in only $O(\mu \sqrt n)$ iterations. Hence we both improve the runtime guarantee in terms of $n$ and we enlarge the range of admissible values for $\mu$. For the smallest admissible population size $\mu = \Theta(\sqrt n \log n)$, we obtain a runtime guarantee of $O(n \log n)$.

From a broader perspective our result yields that the cGA (and we expect similar results to hold for other EDAs) does not suffer from moderate-size valleys of low fitness. We recall that Sudholt and Witt~\cite{SudholtW19} have shown that the cGA with any hypothetical population size (polynomial in $n$) with high probability needs $\Omega(\mu\sqrt n)$ iterations to optimize the \onemax function. Hence our result shows that adding a valley of low fitness to the \onemax function does not worsen the asymptotic performance of the cGA as long as the fitness valley has a width of at most $\frac 1 {20} \ln n -1$.

On the technical side, our work makes some arguments of~\cite{HasenohrlS18} more rigorous. In particular, we observe that the progress of the cGA cannot be estimated by taking the progress one would have when no fitness valley were present and correcting this estimate by inverting the progress with the probability that a search point is sampled in the fitness valley. This argument ignores the stochastic dependencies between the absolute value of the progress and the event that a solution in the fitness valley is sampled. These dependencies are real and have, in fact, a negative impact on the progress as discussed in more detail before Lemma~\ref{ldrift}.

We note that the approach of intentionally ignoring some dependencies to make a mathematical analysis tractable, often called mean-field analysis, is common in some scientific areas, most notably statistical physics, and has also been used in evolutionary computation, e.g.,~\cite{DoerrZ20de}. This approach, however, needs an additional justification, e.g., via specific experiments, why the omission of the dependencies should not change the matter substantially. In any case, such mean-field approaches do not lead to results fully proven with mathematical rigor. In this sense, we hope that our work also provides methods that help in future analyses of EDAs on multimodal optimization problems.

\subsection{An Exponential Lower Bound}

When $k$ is larger, say $k = \omega(\log n)$, then the runtime guarantee given in~\cite{HasenohrlS18} is exponential in $k$, simply because $\mu$ has to be at least exponential in $k$ to fulfill the assumptions of the result. It is clear that with an exponential hypothetical population size, the runtime must be exponential as well (for the sake of completeness, we shall make this elementary argument precise in Lemma~\ref{lem:elem}). What is not immediately clear is if by choosing a smaller hypothetical population size the cGA can optimize jump functions more efficiently. 

\textbf{Our second main result} is a negative answer to this question. In Theorem~\ref{thm:main} we show that, regardless of the hypothetical population size, the runtime of the cGA on a jump function with jump size $k$ with high probability is at least exponential in $k$. Interestingly, not only our result is a uniform lower bound independent of the hypothetical population size, but our proof is also ``uniform'' in the sense that it needs case distinctions neither w.r.t.\ the hypothetical population size nor w.r.t.\ the different reasons for the lower bound. Here we recall that the existing runtime analyses, see, e.g., again~\cite{Droste06,SudholtW19,LenglerSW18}, find two reasons why an EDA can be inefficient. (i) The hypothetical population size is large and consequently it takes long to move the frequencies into the direction of the optimum. (ii) The hypothetical population size is small and thus, in the absence of a strong fitness signal, the random walk of the frequencies brings some frequencies close to the boundaries of the frequency spectrum (this effect is known as \emph{genetic drift}, see~\cite{DoerrZ20tec} for a recent discussion and relatively precise quantification); from there they are hard to move back into the game. 

We avoid such potentially tedious case distinctions via an elegant drift argument on the sum of the frequencies. Ignoring some technicalities here, we show that, regardless of the hypothetical population size, the frequency sum overshoots a value of $n - \frac 14 k$ only after an expected number of $\exp(\Omega(k))$ iterations. However, in an iteration where the frequency sum is below $n - \frac 14 k$, the optimum is sampled only with probability $\exp(-\Omega(k))$. These two arguments prove our lower bound of order $\exp(\Omega(k))$.

\subsection{A Lower Bound for Small Jump Sizes}

Since the exponential lower bound just discussed is not very strong for small jump sizes $k$, we also prove a lower bound of $\Omega(\mu \sqrt n + n \log n)$ for the performance of the cGA with hypothetical population size $\mu = \Omega(\log n)$ on any jump function. This lower bound was shown before for the \onemax function~\cite{SudholtW19}. While it is not surprising that the cGA is not more efficient on jump functions than on \onemax, this is not trivial to show. As we also observe in Section~\ref{sec:nlogn}, a result like ``\onemax is an easiest function with a unique global optimum'', which is true for many other evolutionary algorithms, cannot be proven with the usual arguments for the cGA. In fact, we currently have no indication that such a result is true for the cGA, nor do we have a counter-example.

\subsection{Expected Runtimes of EDAs vs.\ Bounds with High Probability}

As a \textbf{side result}, triggered by the fact that we ``only'' show an upper bound that holds with high probability, but not a bound on the expected runtime, we provide in Section~\ref{sec:whp} a general approach to transform an EDA using a population size parameter $\mu$ into an algorithm that does not require the specification of such a parameter, but has a performance similar to the one of the EDA with optimally chosen parameter. This performance guarantee also holds for the expected runtime, even if for the EDA only a with-high-probability runtime guarantee is known.

\section{Preliminaries}

\subsection{The Compact Genetic Algorithm}

The \emph{compact genetic algorithm} (cGA) is an estimation-of-distribution algorithm (EDA) proposed by Harik, Lobo, and Goldberg~\cite{HarikLG99} for the maximization of pseudo-Boolean functions $\calF : \{0,1\}^n \to \R$. Being a univariate EDA, it develops a probabilistic model described by a \emph{frequency vector} $f \in [0,1]^n$. This frequency vector describes a probability distribution on the search space $\{0,1\}^n$. If $X = (X_1, \dots, X_n) \in \{0,1\}^n$ is a search point sampled according to this distribution---we write \[X \sim \Sample(f)\] to indicate this---then we have $\Pr[X_i = 1] = f_i$ independently for all $i \in [1..n] \coloneqq \{1, \dots, n\}$. In other words, the probability that $X$ equals some fixed search point $y$ is 
\[\Pr[X = y] = \prod_{i : y_i = 1} f_i \prod_{i : y_i = 0} (1 - f_i).\]

In each iteration, the cGA updates this probabilistic model as follows. It samples two search points $x^1, x^2 \sim \Sample(f)$, computes the fitness of both, and defines $(y^1,y^2) = (x^1,x^2)$ when $x^1$ is at least as fit as $x^2$ and $(y^1,y^2) = (x^2,x^1)$ otherwise. Consequently, $y^1$ is the better search point of the two (if not both have the same fitness). We then define a preliminary frequency vector by $f' \coloneqq f + \frac 1 \mu (y^1 - y^2)$, where $\mu$ is an algorithm parameter called  \emph{hypothetical population size}. This definition ensures that, when $y^1$ and $y^2$ differ in some bit position $i$, the $i$-th preliminary frequency moves by a step of $\frac 1 \mu$ into the direction of $y^1_i$, which we hope to be the right direction since $y^1$ is the better of the two search points. The hypothetical population size~$\mu$ controls how strong this update is. 

To avoid a premature convergence, we ensure that the new frequency vector is in $[\frac 1n, 1 - \frac 1n]^n$ by capping too small or too large values at the corresponding boundaries. More precisely, for all $\ell \le u$ and all $r \in \R$ we define 
\[
\minmax(\ell,r,u) \coloneqq \max\{\ell,\min\{r,u\}\} = \begin{cases} 
\ell & \mbox{if $r < \ell$}\\
r & \mbox{if $r \in [\ell,u]$}\\
u & \mbox{if $r > u$}
\end{cases}
\] 
and we lift this notation to vectors by reading it component-wise. Now the new frequency vector is $\minmax(\frac 1n \mathbf{1}_n, f', (1 - \frac 1n) \mathbf{1}_n)$.

This iterative frequency development is pursued until some termination criterion is met. Since we aim at analyzing the time (number of iterations) it takes to sample the optimal solution (this is what we call the \emph{runtime} of the cGA), we do not specify a termination criterion and pretend that the algorithm runs forever.

The pseudo-code for the cGA is given in Algorithm~\ref{alg:cga}. We shall use the notation given there frequently in our proofs. For the frequency vector $f_t$ obtained at the end of iteration $t$, we denote its $i$-th component by $f_{i,t}$ or, when there is no risk of ambiguity, by $f_{it}$. We shall frequently argue with the sum of the frequencies, which can be written as $\|f_t\|_1 := \sum_{i=1}^n |f_{it}|$ since the frequencies are non-negative. With a slight \textbf{abuse of notation}, we extend this common notation also to preliminary frequency vectors $f'$ and thus write $\|f'\|_1 \coloneqq \sum_{i=1}^n f'_{it}$, when there is not danger of confusion. Where there could be a chance of a critical misunderstanding, we use the much less common notation $\sum[v] \coloneqq \sum_{i=1}^n v_i$ to denote the sum of the entries of an $n$-dimensional vector $v \in \R^n$.
	
\begin{algorithm2e}%
	$t \assign 0$\;
	$f_t = (\frac 12, \dots, \frac 12) \in [0,1]^n$\;
	\Repeat{forever}{
    $x^1 \assign \Sample(f_t)$\;
    $x^2 \assign \Sample(f_t)$\;
%    \For{$i \in [1..n]$}{with probability~$\frac 1n$ do $y_i \assign 1 - y_i$\;}
    \leIf{$\calF(x^1) \ge \calF(x^2)$}{$(y^1,y^2) \assign (x^1,x^2)$}{$(y^1,y^2) \assign (x^2,x^1)$}
    $f'_{t+1} \assign f_t + \frac 1 \mu (y^1-y^2)$\;
    $f_{t+1} \assign \minmax(\frac 1n \mathbf{1}_n, f'_{t+1}, (1 - \frac 1n) \mathbf{1}_n)$\;
    $t \assign t+1$\; 
  }
\caption{The compact genetic algorithm (cGA) to maximize a function $\calF : \{0,1\}^n \to \R$.}
\label{alg:cga}
\end{algorithm2e}

\textbf{Well-behaved frequency assumption:} 
For the hypothetical population size $\mu$, we take the common assumption that any two frequencies that can occur in a run of the cGA differ by a multiple of $\frac 1 \mu$. We call this the \emph{well-behaved frequency assumption}. This assumption was implicitly already made in~\cite{HarikLG99} by using even $\mu$ in all experiments (note that the hypothetical population size is denoted by $n$ in~\cite{HarikLG99}). This assumption was made explicit in~\cite{Droste06} by requiring $\mu$ to be even. Both works do not use the frequencies boundaries $\frac 1n$ and $1 - \frac 1n$, so an even value for $\mu$ ensures well-behaved frequencies. 

For the case with frequency boundaries, the well-behaved frequency assumption is equivalent to $(1-\frac 2n)$ being an even multiple of the update step size $\frac 1 \mu$. In this case, $n_\mu = (1 - \frac 2n) \mu \in 2 \N$ and the set of frequencies that can occur is 
\begin{equation}
F \coloneqq F_\mu \coloneqq \{\tfrac 1n + \tfrac i \mu \mid i \in [0..n_\mu]\}.\label{eq:fmu}
\end{equation}
This assumption was made, e.g., in the papers~\cite{FriedrichKKS17} (see the last paragraph of Section~II.C) and~\cite{LenglerSW18} (see the paragraph following Lemma~2.1) as well as in the proof of Theorem~2 in~\cite{SudholtW19}.

\textbf{A trivial lower bound:} We finish this subsection on the cGA with the following very elementary remark, which shows that the cGA with hypothetical population size $\mu$ with probability $1 - \exp(-\Omega(n))$ has a runtime of at least $\min\{\frac{\mu}{4}, \exp(\Theta(n))\}$ on any $\calF : \{0,1\}^n \to \R$ with a unique global optimum (and also on all functions with a sufficiently small exponential number of optima). This shows, in particular, that the cGA with the parameter value $\mu = \exp(\Omega(k))$ used to optimize jump functions with gap size $k \in \omega(\log n) \cap o(n)$ in time $\exp(O(k))$ in~\cite{HasenohrlS18} cannot have a runtime better than exponential in $k$.

\begin{lemma}\label{lem:elem}
  Let $\alpha, \beta \ge 0$ be constants such that $\alpha \beta < \frac 43$. Let $\calF : \{0,1\}^n \to \R$ have at most $\alpha^n$ optima. The probability that the cGA generates an optimum of $\calF$ in $T = \min\{\frac{\mu}{4}, \beta^n\}$ iterations is at most $2 (\alpha\beta \frac 34)^n = \exp(-\Omega(n))$.  
\end{lemma}

\begin{proof}
  %By symmetry, we can assume that the unique global optimum of $f$ is $z = (1, \dots, 1) \in \{0,1\}^n$. 
  By the definition of the cGA, the frequency vector $f$ used in iteration $t = 1, 2, 3, \dots$ satisfies $f \in [\frac 12 - \frac{t-1}{\mu}, \frac 12 + \frac{t-1}{\mu}]^n$. Consequently, the probability that a fixed one of the two search points which are generated in this iteration is a fixed solution, is at most $(\frac 12 + \frac{t-1}{\mu})^n$. For $t \le \frac{\mu}{4}$, this is at most $(\frac 34)^n$. Hence by a simple union bound (over time and the global optima), the probability that an optimum is generated in the first $T = \min\{\frac{\mu}{4}, \beta^n\}$ iterations, is at most $2 \alpha^n T (\frac 34)^n \le 2 (\alpha\beta \frac 34)^n = \exp(-\Omega(n))$.
\end{proof}

\subsection{Runtime Analysis for the cGA}\label{sec:prevcga}

In this subsection, we brief{}ly describe the relevant previous runtime analyses for the cGA. For simplicity, we shall always assume that the hypothetical population size is at most polynomial in the problem size $n$, that is, that there is a constant $c$ such that $\mu \le n^c$. This is justified, among others, by Lemma~\ref{lem:elem}, which shows that a super-polynomial hypothetical population size immediately leads to a super-polynomial runtime on any objective function with at most $\alpha^n$ optima, where $\alpha$ can be any constant less than $\frac 43$. 

The first to conduct a rigorous runtime analysis for the cGA was Droste in his seminal work~\cite{Droste06}. He regarded the cGA without frequency boundaries, that is, he just took $f_{t+1} \coloneqq f'_{t+1}$ in our notation. He showed that this algorithm with $\mu \ge n^{1/2 + \eps}$, $\eps > 0$ any positive constant, finds the optimum of the $\onemax$ function defined by 
\[\onemax(x) = \|x\|_1 = \sum_{i=1}^n x_i\] 
for all $x \in \{0,1\}^n$ with probability at least $\frac 12$ in $O(\mu \sqrt n)$ iterations~\cite[Theorem~8]{Droste06}. 

Droste also showed that this cGA for any objective function $\calF$ with unique optimum has an expected runtime of $\Omega(\mu \sqrt n)$ when conditioning on no premature convergence~\cite[Theorem~6]{Droste06}. It is easy to see that his proof of the lower bound can be extended to the cGA with frequency boundaries, that is, to Algorithm~\ref{alg:cga}. For this, it suffices to deduce from his drift argument the result that the first time $T_{n/4}$ that the frequency distance $D = \sum_{i=1}^n (1 - f_{it})$ is less than $\frac n4$ satisfies $E[T_{n/4}] \ge \frac{\sqrt 2}{4}  \mu \sqrt n $. Since the probability to sample the optimum from a frequency distance of at least $\frac n4$ is at most $\exp(-\frac n4)$, see Lemma~\ref{loptUB}, the algorithm with high probability does not find the optimum before time~$T_{n/4}$.

Around ten years after Droste's work, Sudholt and Witt~\cite{SudholtW19} showed that the $O(\mu \sqrt n)$ upper bound also holds for the cGA with frequency boundaries. There (but the same should be true for the cGA without boundaries) a hypothetical population size of $\mu =\Omega(\sqrt n \log n)$ suffices (recall that Droste required $\mu = \Omega(n^{1/2+\eps})$). The technically biggest progress with respect to upper bounds most likely lies in the fact that the analysis in~\cite{SudholtW19} also holds for the expected optimization time, which means that it also includes the rare case that frequencies reach the lower boundary (see our discussion of the relation of expectations and tail bounds for runtimes of EDAs in Section~\ref{sec:whp}). Sudholt and Witt also show that the cGA with frequency boundaries with high probability (and thus also in expectation) needs at least $\Omega(\mu\sqrt n + n \log n)$ iterations to optimize $\onemax$. While the $\Omega(\mu\sqrt n)$ lower bound could have been also obtained with methods similar to Droste's (in Lemma~\ref{lonemax} we do something very similar), the innocent-looking $\Omega(n \log n)$ bound is surprisingly difficult to prove.

Not much is known for hypothetical population sizes below the order of $\sqrt n$. It is clear that then the frequencies will reach the lower boundary of the frequency range, so working with a non-trivial lower boundary like~$\frac 1n$ is necessary to prevent premature convergence. The recent lower bound $\Omega(\mu^{1/3} n)$ valid for $\mu = O(\frac{\sqrt n}{\log n \log\log n})$ of~\cite{LenglerSW18} indicates that already a little below the $\sqrt n$ regime significantly larger runtimes occur, but with no upper bounds this regime remains largely not understood.

We refer the reader to the recent survey~\cite{KrejcaW20bookchapter} for more results on the runtime of the cGA on classic unimodal test functions like \leadingones and \binval. Interestingly, nothing was known for multimodal functions before the recent work of Hasen\"ohrl and Sutton~\cite{HasenohrlS18} on jump functions, which we discussed already in the introduction. 

The general topic of lower bounds on runtimes of EDAs remains largely little understood. Apart from the lower bounds for the cGA on \onemax discussed above, the following is known. Krejca and Witt~\cite{KrejcaW20} prove a lower bound for the UMDA on \onemax, which is of a similar flavor as the lower bound for the cGA of Sudholt and Witt~\cite{SudholtW19}: For $\lambda = (1 + \beta) \mu$, where $\beta > 0$ is a constant, and $\lambda$ polynomially bounded in $n$, the expected runtime of the UMDA on \onemax is $\Omega(\mu \sqrt n + n \log n)$. For the binary value function \binval, Droste~\cite{Droste06} and Witt~\cite{Witt18} together give a lower bound of $\Omega(\min\{n^2, \mu n\})$ for the runtime of the cGA. Apart from these sparse results, we are not aware of any lower bounds for EDAs. Of course, the black-box complexity of the problem is a lower bound for any black-box algorithm, hence also for EDAs, but these bounds are often lower than the true complexity of a given algorithm. For example, the black-box complexities of \onemax, \leadingones, and jump functions with jump size $k \le \frac 12 n - n^{\eps}$, $\eps > 0$ any constant, are $\Theta(\frac{n}{\log n})$~\cite{DrosteJW06,AnilW09}, $\Theta(n \log\log n)$~\cite{AfshaniADDLM19}, and $\Theta(\frac{n}{\log n})$~\cite{BuzdalovDK16}, respectively.

\subsection{Runtime Results for Jump Functions}\label{ssec:prevjump}

To complete the picture, we briefly describe some typical runtimes of evolutionary algorithms on jump functions. We recall that the $n$-dimensional jump function with jump size $k \ge 1$ is defined by
\[
\jump_{nk}(x) = 
\begin{cases}
\|x\|_1+k & \mbox{if $\|x\|_1 \in [0..n-k] \cup \{n\}$,}\\
n - \|x\|_1 & \mbox{if $\|x\|_1 \in [n-k+1\, ..\, n-1]$}.
\end{cases}
\]
Hence for $k = 1$, we have a fitness landscape identical to the one of $\onemax$ apart from all fitness values being larger by one. For larger values of $k$, we still have a fitness landscape identical to \onemax apart from constant shifts when only regarding the lowest $n-k$ fitness levels of the \onemax function, however, now there is a fitness valley (``gap'')
\begin{equation}
G_{nk} \coloneqq \{x \in \{0,1\}^n \mid n-k < \|x\|_1 < n\}\label{eq:gnk}
\end{equation}
consisting of the $k-1$ highest sub-optimal fitness levels of the \onemax function. 

This valley is hard to cross via standard-bit mutation with mutation rate~$\frac 1n$. Consequently, as proven in the classic paper~\cite{DrosteJW02}, the \oea has an expected optimization time of at least $n^k$ on $\jump_{nk}(x)$. This lower bound also holds for the \mplea for all values of $\mu$ and $\lambda$ as well as, more surprisingly, for the \mclea for large ranges of the population sizes~\cite{Doerr20gecco}. By using larger mutation rates or a heavy-tailed mutation operator, a $k^{\Theta(k)}$ runtime improvement for the runtime of the \oea can be obtained~\cite{DoerrLMN17}, but the runtime remains $\Omega(n^k)$ for $k$ constant (and this is also true for the variation of the heavy-tailed mutation rate proposed in~\cite{FriedrichQW18}). The runtime stemming from the optimal mutation rate can be automatically obtained (apart from constant factors) via a self-adjusting choice of the mutation rate~\cite{RajabiW20}. 

Asymptotically better runtimes can be achieved when using crossover, though this is harder than expected. The first work in this direction~\cite{JansenW02}, among other results, could show that a simple $(\mu+1)$ genetic algorithm using uniform crossover with rate $p_c = O(\frac{1}{kn})$ obtains an $O(\mu n^2 k^3 + 2^{2k} p_c^{-1})$ runtime when the population size is at least $\mu = \Omega(k \log n)$. A shortcoming of this result, already noted by the authors, is that it only applies to uncommonly small crossover rates. Using a different algorithm that first always applies crossover and then mutation, a runtime of $O(n^{k-1} \log n)$ was achieved by Dang et al.~\cite[Theorem~2]{DangFKKLOSS18}. For $k \ge 3$, the logarithmic factor in the runtime can be removed by using a higher mutation rate. With additional diversity mechanisms, the runtime can be further reduced up to ${O(n \log n + 4^k)}$, see~\cite{DangFKKLOSS16}. In the light of this last result, the insight stemming from the previous work~\cite{HasenohrlS18} and ours is that the cGA apparently without further modifications supplies the necessary diversity to obtain a runtime of $O(n \log n + 2^{O(k)})$.

With a three-parent majority vote crossover, among other results, a runtime of $O(n \log n)$ could be obtained via a suitable island model for all $k = O(n^{\frac 12 - \eps})$~\cite{FriedrichKKNNS16}. Via a hybrid genetic algorithm using as variation operators only local search and a deterministic voting crossover, an $O(n)$ runtime for $m = O(\log n)$ was obtained in~\cite{WhitleyVHM18}. Via a different voting mechanism, an $O(n \log n)$ runtime was obtained even for $m$ as large as $O(n)$~\cite{RoweA19}. With the right static or heavy-tailed parameters, the \ollga optimizes jump functions in time roughly $n^{(k+O(1))/2}$~\cite{AntipovDK20,AntipovD20ppsn}, however, when using the parameterization developed for \onemax~\cite{DoerrDE15}, then several self-adjusting versions of the \ollga cannot beat mutation-based EAs as shown in~\cite{FajardoS20}.

Finally, we note that runtimes of $O(n \binom{n}{k})$ and $O(k \log(n) \binom{n}{k})$ were shown for the $(1+1)$~IA$^{\mathrm hyp}$ and the $(1+1)$ Fast-IA artificial immune systems, respectively~\cite{CorusOY17,CorusOY18fast}.

\subsection{Expected Runtimes versus Guarantees with High Probability}\label{sec:whp}

We note that our main upper bound result as well as the previous one~\cite{HasenohrlS18} for this problem give runtime bounds that hold with high probability, that is, with probability $1 - o(1)$. However, we do not show a bound on the expected runtime. Let us quickly argue what the differences are, why we chose to prove a high-probability statement, and how to transform EDAs with high-probability guarantees into EDAs with guarantees on the expected runtime. We note that Wegener~\cite[Section~3]{Wegener05} with different arguments also suggests to prefer high-probability guarantees over expected runtimes.

For most evolutionary algorithms a high-probability guarantee can easily be turned into a bound on the expected runtime. If we know that a certain algorithm from any initial state finds the optimum in time $T$ with at least constant probability, then by splitting time into consecutive segments of length $T$ we see that after time $\gamma T$ the probability that the algorithm has not succeeded is at most $\exp(-\Omega(\gamma))$. Consequently, the runtime is stochastically dominated (see Section~\ref{ssec:toolscga} for the definition of this notation) by $T$ times a geometric random variable with constant success rate, and consequently, the expected runtime is $O(T)$. The same argument gives a scalable tail bound of type ``for all $\gamma>1$, the probability that the runtime is more than $\gamma T$ is at most $\exp(-\Omega(\gamma))$.''

For EDAs, it is usually much harder to show a good performance for any initial situation since there are some states which are particularly unfavorable (usually when all frequencies are close to the wrong boundary value). This does not rule out that the expected runtime and the time that is obtained with high probability are of the same order, but proving the bound on the expected runtime needs stronger arguments. The analysis of the expected runtime of the cGA on \onemax in~\cite{SudholtW19} is an example for such a result. 

This additional proof complexity raises the question if this effort is justified if the hardest part is dealing with states of the algorithm that are rarely reached (in~\cite{SudholtW19} with probability $O(n^{-c})$ only, where $c$ can be any positive constant). While we think that it was very valuable that the work~\cite{SudholtW19} showed how to compute expected runtimes for EDAs, we feel that such results are not always needed, both because of the difficulty to obtain such results and because, in some sense, they are a mildly unnatural remedy to the deeper problem.

As said, the main reason why guarantees for the expected runtime of an EDA can be difficult to show is that the EDA with small probability can end up in a state from which the optimum is hard to reach. When in such a state, however, instead of spending much time to leave the unfavorable state, it would be more efficient and more natural to simply restart the algorithm and have a new good chance for a fast optimization process. While we cannot expect the algorithm to detect that it is in an unfavorable state, the following simple parallel-run strategy under mild assumptions can do this automatically. More precisely, via suitable parallel runs we obtain an expected runtime that is only a logarithmic factor above the runtime the EDA would have with high probability when using the optimal population size. Hence this approach both obtains expected runtimes and optimizes the value of the parameter $\mu$. We note that the ``noise-oblivious scheme'' proposed in~\cite[Algorithm~4]{FriedrichKKS17} can also be used to optimize the parameter $\mu$, however only under the much stronger assumption that the runtime (or an upper bound, which influences the runtime of the scheme) is known. In this case, a simple restart scheme with multiplicatively increasing $\mu$ values does the job.

We now proceed with detailing our parallel-run strategy. In the remainder, we shall assume the following.\\

\noindent\textbf{General assumption:} Let $\calA$ be an EDA (or any other randomized search heuristic) with a parameter $\mu$ and let $\calP$ be a problem instance we want to solve. We assume that there are unknown values $\tilde\mu$ and $T$ such that $\calA$ with any parameter value $\mu \ge \tilde\mu$ solves $\calP$ in time $\mu T$ with probability at least $\frac 34$. \\

For this situation, we proposed the following strategy.\\

\noindent\textbf{Parallel EDA runs with exponentially growing population size:} We propose the following strategy to solve $\calP$ via parallel runs of $\calA$ with different parameter values. We start with no process running. In round $i = 1, 2, \dots$ of our strategy, we let all running processes (which are process $1$ to $i-1$) use a computational budget of $2^{i-1}$; further, we start process $i$ with parameter $\mu =  2^{i-1}$ and let it use a budget of $\sum_{j=0}^{i-1} 2^j$. These processes can be run in parallel or sequentially in any order. The pseudocode for this strategy is given in Algorithm~\ref{alg:parallel}.\\

\begin{algorithm2e}
 Initialize process 1 with population size $\mu=1$ and run it for one generation\;
 \For{$i=2,3,\dots$}{
   Run processes $1,\dots, i-1$ each for another $2^{i-1}$ generations\;
   Initialize process $i$ with population size $\mu=2^{i-1}$ and run it for $\sum_{j=0}^{i-1}2^j$ generations\; 
   }
\caption{The parallel-run cGA to maximize a function $f: \{0,1\}^n \rightarrow \R$}\label{alg:parallel}
\end{algorithm2e}

\noindent\textbf{Analysis:}
We observe that at the end of round $i$, processes $1$ to $i$ are running and have each spent a budget of $\sum_{j=0}^{i-1} 2^j = 2^i -1$ up to this point in time. Consequently, the total budget spent in the first $i$ rounds is less than $i 2^i$. 

Note that after round $i_0 \coloneqq 1 + \lceil \log_2 \tilde\mu \rceil + \lfloor \log_2 T \rfloor$, the process started with parameter value $\mu = \mu_0 \coloneqq 2^{\lceil \log_2 \tilde\mu \rceil} \ge \tilde\mu$ has started and has used a time budget of 
\[\sum_{j=0}^{i_0 - 1} 2^j \ge \sum_{j=\lceil \log_2 \tilde\mu \rceil}^{i_0 - 1} 2^j = \mu_0 \sum_{j=0}^{\lfloor \log_2 T \rfloor} 2^j \ge \mu_0 T.\] 
Consequently, with probability $\frac 34$ this process has found the optimum at that time. With the same type of computation, we see that after round $i_0 + j$, the process with parameter value $\mu = 2^j \mu_0$ is finished with probability $\frac 34$. Consequently, the round in which we find the solution is stochastically dominated (see Section~\ref{ssec:toolscga}) by $i_0 - 1$ plus a geometric distribution (on $1, 2, \ldots$) with success rate $\frac 34$. The expected time taken by this strategy to solve $\calP$ thus is at most 
\[\sum_{i = i_0}^\infty \left(\frac 14\right)^{i - i_0} \left(\frac 34\right) i \, 2^i = \frac 34 \, 2^{i_0} \sum_{j=0}^\infty 2^{-j} (j+i_0) = 3 \cdot 2^{i_0-1} (i_0+1)\]
using the well-known equality $\sum_{j=0}^\infty j \, 2^{-j} = 2$. We continue estimating the expected runtime of our parallel-run strategy by
\[
3 \cdot 2^{i_0-1} (i_0+1) \le 6 \tilde\mu T (\log_2 (\tilde\mu T) + 3) \eqqcolon T_{\paral}.
\]
We note that if the values of $\tilde \mu$ and $T$ were known in advance, then restarting the EDA with $\mu = \tilde \mu$ and with a budget of $T$ until the problem is solved would immediately give an algorithm with expected runtime at most $T^* = \frac 43 \tilde \mu T$. This is the best-possible expected runtime that can be deduced from our assumptions. Consequently, our parallel-run strategy with its $O(T^* \log T^*)$ expected runtime obtains the optimal expected runtime apart from a logarithmic factor. 

In summary, we have shown the following result.

\begin{theorem}\label{tscheme}
  Under the general assumptions made above, with $i_0 \coloneqq 1 + \lceil \log_2 \tilde\mu \rceil + \lfloor \log_2 T \rfloor$, the parallel run strategy described above has the following performance.
  \begin{itemize}
  \item The expected time until $\calP$ is solved is at most 
  \[3 \cdot 2^{i_0-1} (i_0+1) \le 6 \tilde\mu T (\log_2 (\tilde\mu T) + 3) = \tfrac 92 T^* (\log_2 (\tfrac 34 T^*) + 3),\] where $T^* = \frac 43 \tilde \mu T$ is the best expected runtime that can be achieved via restarts of $\calA$ under the general assumptions.  
  \item For all $j = 0, 1, 2, \ldots$, the probability that a runtime of $2^{i_0+j} (i+j)$ does not suffice to solve $\calP$, is at most $4^{-j-1}$.
  \end{itemize}
\end{theorem}

We remark that a logarithmic factor performance loss over the optimal strategy (requiring the precise values of $\tilde\mu$ and $T$) is not a lot compared to what can be lost by choosing a wrong algorithm parameter, in particular, when the parameter is hard to guess. We note here that the recent work~\cite{LenglerSW18} suggests that already for the simple \onemax function, the hypothetical population size has a non-obvious influence on the runtime: Sufficiently small values give an $O(n \log n)$ runtime, in a middle regime the runtime increases to $\tilde\Omega(n^{7/6})$ before dropping again to $O(n \log n)$ and then increasing linearly with $\mu$. In the light of such results, a logarithmic overhead for automatically finding a near-optimal rate appears to be a good trade-off.

Finally, we remark without further proof that when our general assumption is fulfilled with some failure probability $p$ instead of $\frac 14$, then tail probabilities in the second item of Theorem~\ref{tscheme} are of order $p^{j+1}$ instead of $(\frac 14)^{j+1}$. This could potentially be interesting when the performance of $\calA$ is strongly concentrated so that the general assumptions hold with some $p = o(1)$. We also note that our strategy could be adjusted to deal with smaller success probabilities than $\frac 34$, either by increasing the $\mu$ value by a smaller factor than $2$ or by having several processes using the same $\mu$ value. We spare the details.

Finally, we note that recently a similar approach was proposed in~\cite{DoerrZ20gecco}. The main difference to ours is that runs where stopped after a time that was based on a mathematical analysis of when genetic drift could become problematic. From the implementation point of view, in this approach the runs with different values of $\mu$ can be conducted one after the other. From the theoretical perspective, this approach has the advantage that with the right choice of the hyperparameters the $\Theta(\log(\tilde \mu T))$ factor in the runtime bound of Theorem~\ref{tscheme} can be saved. The experimental results in~\cite{DoerrZ20gecco} suggest that their approach is superior when the hyperparameters are chosen suitably, which is however non-trivial. We note that here that there is a general agreement in the community that genetic drift leads to an undesired behavior of the EDA. Genetic drift can lead to catastrophic runtimes (compare the results of~\cite{LehreN19foga,DoerrK20evocop}), but not always does (\cite{LehreN17,Witt19} show that the UMDA already with a population size of $\Theta(\log n)$ and thus clearly in the genetic drift regime can optimize \onemax in the for this algorithm best known runtime $O(n \log n)$).

\section{Technical Tools}

In this section, we collect a number of technical results that will be used in our main proofs. These include standard arguments like elementary estimates, Chernoff bounds, and drift theorems, as well as original arguments for the analysis of the cGA which might be of general interest such as a tool to quantify the effect of frequencies being capped at the boundaries (Lemma~\ref{lboundary}), an upper and a lower bound for the probability of sampling the optimum given the $\ell_1$-distance between the current frequency vector and the optimum (Lemma~\ref{loptUB} and~\ref{lopt}), an estimate for the time taken to sample a search point close to the current frequency vector, and a lower bound on the probability to sample two different search points in one iteration (Lemma~\ref{ldiff}).

\subsection{Standard Tools}

The following estimate seems well-known (e.g., it was used in~\cite{JansenJW05} without proof or reference). Gie{\ss}en and Witt~\cite[Lemma~3]{GiessenW17} give a proof via estimates of binomial coefficients and the binomial identity. A more elementary proof can be found in~\cite[Lemma~1.10.37]{Doerr20bookchapter}.

\begin{lemma}\label{lprobbino}
  Let $X \sim \Bin(n,p)$. Let $k \in [0..n]$. Then \[\Pr[X \ge k] \le \binom{n}{k} p^k.\]
\end{lemma}

We regularly use the following well-known \emph{multiplicative Chernoff bounds}, which can be derived from~\cite{Hoeffding63}, see, e.g., Theorems~1.10.1 and 1.10.5 together with Section 1.10.1.8 in~\cite{Doerr20bookchapter}. 

\begin{theorem}\label{tchernoff}
  Let $X_1, \ldots, X_n$ be independent random variables taking values in $[0,1]$. Let $X = \sum_{i = 1}^n X_i$. Let $\mu^+ \ge E[X]$ and $\mu^- \le E[X]$. Let $\delta \ge 0$ and $\tilde\delta \in [0,1]$. Then 
  \begin{align*}
  \Pr[X \ge (1+\delta) \mu^+] &\le \exp\bigg(-\frac{\min\{\delta^2,\delta\} \mu^+}{3}\bigg),\\
  \Pr[X \le (1-\tilde\delta) \mu^-] &\le \exp\bigg(-\frac{\tilde\delta^2 \mu^-}{2}\bigg).
  \end{align*}
\end{theorem}

A direct consequence of these Chernoff bounds are the following estimates, which state that the \onemax fitness of a search point sampled from $\Sample(f)$ is close to the expected \onemax fitness $\|f\|_1$. Since we mostly need such results for frequency vectors close to $(1, \ldots, 1)$, we formulate this result in terms of distances to the maximum value $n$. 

\begin{lemma}\label{lsample}
  Let $f \in [0,1]^n$, $D \coloneqq n - \|f\|_1$, $D^- \le D \le D^+$, $x \sim \Sample(f)$, and $d(x) \coloneqq n - \|x\|_1$. Then for all $\delta \ge 0$ and $\tilde\delta \in [0,1]$, we have
  \begin{align*}
  \Pr[d(x) \ge (1+\delta) D^+] & \le \exp(-\tfrac 13 \min\{\delta^2,\delta\} D^+),\\
  \Pr[d(x) \le (1-\tilde\delta) D^-] & \le \exp(-\tfrac 12 \tilde\delta^2 D^-).
  \end{align*}
\end{lemma}

\begin{proof}
  The random variable $n-\|x\|_1$ can be written as a sum $n-\|x\|_1 = \sum_{i=1}^n Z_i \eqqcolon Z$ of $n$ independent binary random variables $Z_1, \dots, Z_n$ such that $\Pr[Z_i = 1] = 1 - f_{i}$. By definition, $E[Z] = D$. The claims follow directly from Theorem~\ref{tchernoff}.
\end{proof}

We need the lemma above in particular to argue that the probability to sample a search point in the gap region of the $\jump$ function is small. For the $\jump_{nk}$ function, we observe that when $D \coloneqq n - \|f\|_1$ is at least $2k$, then the probability that $x \sim \Sample(f)$ lies in the gap, that is, satisfies $n-k < \|x\|_1 < n$, is $e^{-\Omega(k)}$. This result is sufficient for our purposes. We note that we could also obtain a low constant probability for sampling in the gap when $D \ge k + \Omega(\sqrt k)$ with large implicit constant. In~\cite[Lemma~3.2]{HasenohrlS18}, a gap probability of at most $1 - \frac{1}{\sqrt 2} \le 0.293$ is claimed already when $D \ge k+c$ for $c$ a sufficiently large constant and $k = o(n)$, but we are skeptical that this is true. Note that when $f = \frac{n-k-c}{n} \mathbf{1}_n$, then $X = n - \|x\|_1$ with $x \sim \Sample(f)$ follows a binomial distribution with parameters $n$ and $\frac{k+c}{n}$. Hence if $k$ is large compared to $c$, then $\Pr[X < k] = \Pr[X < E[X] - c] \approx \frac 12$. 

At one point, in the proof of Lemma~\ref{lconc}, we need an additive Chernoff bound not only for the sum of independent random variables, but also for all partial sums. Such bounds are less known despite the fact that many classical Chernoff bounds hold equally well in this more demanding fashion. The following result is from Hoeffding~\cite[Theorem~2 together with~(2.17)]{Hoeffding63}. It can also be found in~\cite{Doerr20bookchapter}, Theorems~1.10.9 and~1.10.31.

\begin{theorem}\label{tchernoffmax}
  Let $X_1, \ldots, X_n$ be independent random variables such that for all $i \in [1..n]$, the variable $X_i$ takes values in some interval $[a_i,b_i]$ and has expectation $E[X_i] = 0$. Then for all $\lambda \ge 0$, we have
  \begin{align*}
  &\Pr\left[\exists j \in [1..n] : \sum_{i=1}^j X_i \ge \lambda\right] \le \exp\left(-\frac{2 \lambda^2}{\sum_{i=1}^n (b_i-a_i)^2}\right),\\
  &\Pr\left[\exists j \in [1..n] : \sum_{i=1}^j X_i \le -\lambda\right] \le \exp\left(-\frac{2 \lambda^2}{\sum_{i=1}^n (b_i-a_i)^2}\right).
  \end{align*}
\end{theorem}
  
Finally, we state the additive drift theorem of He and Yao~\cite{HeY01} (see also the recent survey~\cite{Lengler20bookchapter}), which allows to translate an expected progress (or bounds on it) into bounds for expected hitting times.   
  
\begin{theorem}\label{tdrift}
  Let $S \subseteq \R_{\ge 0}$ be finite and $0 \in S$. Let $X_0, X_1, \ldots$ be a random process taking values in $S$. Let $\delta > 0$. Let $T = \inf\{t \ge 0 \mid X_t = 0\}$.
  \begin{enumerate}
  \item \label{it:driftub} If for all $t \ge 0$ and all $s \in S \setminus \{0\}$ we have $E[X_t - X_{t+1} \mid X_t = s] \ge \delta$, then $E[T] \le \frac{E[X_0]}{\delta}$.
  \item \label{it:driftlb} If for all $t \ge 0$ and all $s \in S \setminus \{0\}$ we have $E[X_t - X_{t+1} \mid X_t = s] \le \delta$, then $E[T] \ge \frac{E[X_0]}{\delta}$.
  \end{enumerate}
\end{theorem}

\subsection{Tools for the Analysis of the cGA}\label{ssec:toolscga}

In this section, we prove a number of general arguments for the analysis of the cGA. Since we expect that they are helpful for other runtime analyses of EDAs, we fix no general notation apart from the one defined in Algorithm~\ref{alg:cga} (at the price of occasionally restating a notation). 

We recall the notation of stochastic domination, which will be used several times in this work. For two random variables $X$ and $Y$, not necessarily defined over the same probability space, we say that $Y$ \emph{stochastically dominates}~$X$, written as $X \preceq Y$, if for all $\lambda \in \R$ we have 
\[\Pr[X \ge \lambda] \le \Pr[Y \ge \lambda].\]
Stochastic domination is a strong way of saying that $Y$ is not smaller than $X$. It implies that $E[X] \le E[Y]$. We refer to~\cite{Doerr19tcs} for more details.

\textbf{Boundary effects:}
When, in the notation of Algorithm~\ref{alg:cga}, the current frequency vector $f_t$ is such that $f_{it} \in \{\frac 1n, 1 - \frac 1n\}$ for some $i \in [1..n]$, then it may happen that $f'_{t+1} \notin [\frac 1n, 1-\frac 1n]$ and consequently $f_{t+1}$ does not satisfy the nice relation $f_{t+1} = f_t + \frac 1 \mu (y^1 - y^2)$. The following lemma quantifies these discrepancies. We here recall the common definition that for an $n$-dimensional vector $x$ and a subset $L \subseteq [1..n]$ of its index set, $x_{|L}$ denotes the restriction of $x$ to $L$, that is, the vector $(x_\ell)_{\ell \in L}$.
 
\begin{lemma}\label{lboundary}
  Let $P = 2 \frac 1n (1-\frac 1n)$. Let $t \ge 0$. Using the notation given in Algorithm~\ref{alg:cga}, consider iteration $t+1$ of a run of the cGA started with a fixed frequency vector $f_t \in [\frac 1n, 1-\frac 1n]^n$. 
  \begin{enumerate} 
  \item\label{it:boundaryL} Let $L = \{i \in [1..n] \mid f_{it} = \frac 1n\}$, $\ell = |L|$, and $M = \{i \in L \mid x^1_i \neq x^2_i\}$. Then $|M| \sim \Bin(\ell,P)$ and 
  \[\|f_{t+1}\|_1 - \|f'_{t+1}\|_1 \preceq \|(f_{t+1})_{|L}\|_1 - \|(f'_{t+1})_{|L}\|_1 \preceq \tfrac 1\mu |M| \preceq \tfrac 1\mu \Bin(n,\tfrac 2n).\] 
  \item\label{it:boundaryU} Let $L = \{i \in [1..n] \mid f_{it} = 1 - \frac 1n\}$, $\ell = |L|$, and $M = \{i \in L \mid x^1_i \neq x^2_i\}$. Then $|M| \sim \Bin(\ell,P)$ and \[\|f'_{t+1}\|_1 - \|f_{t+1}\|_1 \preceq \|(f'_{t+1})_{|L}\|_1 - \|(f_{t+1})_{|L}\|_1 \preceq \tfrac 1\mu |M| \preceq \tfrac 1\mu \Bin(n,\tfrac 2n).\]
  \end{enumerate} 
\end{lemma}

\begin{proof}
By symmetry, it suffices to prove the first part. For an $i \in L$, we have $\Pr[x^1_i \neq x^2_i] = 2 \frac 1n (1-\frac 1n) = P$. Since the bits of $x^1$ and $x^2$ were sampled independently, we have $|M| \sim \Bin(\ell,P)$.

By the well-behaved frequency assumption and the fact that $f'_{t+1} = f_t + \frac{1}{\mu} (y^1 - y^2)$ for binary vectors $y^1$ and $y^2$, we can have $f'_{i,t+1} < \frac 1n$ and thus $f_{i,t+1} > f'_{i,t+1}$ only when $f_{it} = \frac 1n$ and $x^1_i \neq x^2_i$, that is, when $i \in M$. This shows $\|f_{t+1}\|_1 - \|f'_{t+1}\|_1 \preceq \|(f_{t+1})_{|L}\|_1 - \|(f'_{t+1})_{|L}\|_1$. 

Since $f_{i,t+1} > f'_{i,t+1}$ implies $f_{i,t+1} = f'_{i,t+1} + \frac 1\mu$, we also have $\|(f_{t+1})_{|L}\|_1 - \|(f'_{t+1})_{|L}\|_1 \preceq \frac 1\mu |M| \preceq \tfrac 1\mu \Bin(n,\tfrac 2n)$.  
\end{proof}

\textbf{Sampling a particular solution:}
The following two elementary estimates give an upper and a lower bound on the probability to sample a particular search point $x^*$. Note that the quantity $D = \|x^* - f\|_1$ is our usual distance measure $D = n - \|f\|_1$ when $x^* = (1, \dots, 1)$. 

\begin{lemma}\label{loptUB}
  Let $x^* \in \{0,1\}^n$. Let $f \in [0,1]^n$ and $D = \|x^* - f\|_1$. Let $x \sim \Sample(f)$. Then $\Pr[x = x^*] \le \exp(-D)$.
\end{lemma}

\begin{proof}
  The probability to sample $x^*$ is 
  \begin{align*}
    \prod_{i=1}^n (1 &- |x^*_i - f_{it}|) \le \prod_{i=1}^n \exp(-|x^*_i - f_{it}|) \\
    &= \exp\left(-\sum_{i=1}^n |x^*_i - f_{it}|\right) = \exp(-\|x^* - f\|_1) = \exp(-D).%\qedhere
  \end{align*}
\end{proof}

To ease reading, we formulate the following estimate only for $x^* = (1, \dots, 1)$, but it is clear that by symmetry analoguous statements hold for arbitrary $x^*$ (when $\|x^* - f\|_\infty \le 1 - c$).
%note: die entsprechende aussage in HS18 ist $\Pr[...] \ge c^{-n}$. Das ist schwaecher, aber nicht so schlimm, da mit Frequenzen gearbeitet wird, die seeeehr nah am gap sind (was natuerlich ein groesses \mu erfordert)
\begin{lemma}\label{lopt}
  Let $0 < c < 1$, $f \in [c,1]^n$, and $D = n - \|f\|_1$. Let $x \sim \Sample(f)$. Then $\Pr[x = (1, \dots, 1)] \ge c^{D/(1-c)}$.
\end{lemma}

\begin{proof}
  For $i \in [1..n]$, let $\alpha_i \coloneqq \frac{1-f_i}{1-c\phantom{_i}}$. Then $f_i = \alpha_i c + (1-\alpha_i) 1$ is the unique representation of $f_i$ as convex combination of $c$ and $1$. Since the logarithm is concave, we have 
  \[\log f_i = \log(\alpha_i c + (1-\alpha_i) 1) \ge \alpha_i \log c + (1-\alpha_i) \log 1 = \log(c^{\alpha_i}).\] 
  Since the logarithm is monotonically increasing, this inequality implies ${f_i \ge c^{\alpha_i}}$. Consequently, 
  \begin{align*}
  \Pr[x = (1, \dots, 1)] & = \prod_{i=1}^n f_i  \ge \prod_{i=1}^n c^{\alpha_i} = c^{\sum_{i=1}^n \alpha_i} = c^{(n - \|f\|_1)/(1-c)}.
  \end{align*}
\end{proof}

\textbf{Time to sample a search point when close:}
We use the above lower bound on the probability to sample $(1, \dots, 1)$ to prove the following result. It shows that when $\mu$ is large enough and $D_t \coloneqq n - \|f_t\|_1$ is small enough, then regardless of the fitness function we sample the search point $(1, \ldots, 1)$ quickly with high probability. The main argument is that when $\mu$ is sufficiently large, then $D_t$ stays small sufficiently long.

\begin{lemma}\label{lfinish1}
  Let $\mu$ be at least $\sqrt n$, but polynomially bounded in $n$. Consider a run of the cGA with hypothetical population size $\mu$ on an arbitrary fitness function. Assume that at some time $t_0$, we have $D_{t_0} := n - \|f_{t_0}\|_1 \le \frac 1 {10} \ln n$ and $f_{t_0} \in [\frac 13, 1]^n$. Then with probability at least $1 - n^{-\omega(1)}$, the search point $(1, \dots, 1)$ is sampled in the next $\frac{D_{t_0} \mu}{2 \ln(n)^2}$ iterations.
\end{lemma}

\begin{proof}
  We first argue that if at some time $t$ we have $D_t \le \ln(n)^2$, then $\Pr[D_{t+1} \ge D_t + \frac 2\mu \ln(n)^2] \le n^{-\omega(1)}$. By Lemma~\ref{lsample}, we have 
  \[\Pr[d(x^j) \ge 2\ln(n)^2] \le \exp(-\tfrac 13 \ln(n)^2) = n^{-\omega(1)}\]
   for $j = 1,2$. Consequently, with probability $1 - n^{-\omega(1)}$, we have both $\|x^1\|_1 > n - 2\ln(n)^2$ and $\|x^2\|_1 > n - 2\ln(n)^2$. Now regardless of how $x^1$ and $x^2$ are sorted into $(y^1,y^2)$, less than $2 \ln(n)^2$ frequencies are decreased in the frequency update of this iteration. We conclude that $D_{t+1} < D_t + \frac 2\mu \ln(n)^2$.
  
  Let $L = \lfloor \frac{D_{t_0}}{(2/\mu) \ln(n)^2} \rfloor$. By a union bound, with probability 
  \[1 - Ln^{\omega(1)} \ge 1 - n^{\omega(1)},\] we have $D_{t+1} \le D_t + \frac 2\mu \ln(n)^2$ in all iterations $t = t_0, \dots, t_0 + L -1$ that start with $D_t \le 2 D_{t_0}$. Let us condition on this in the following. Then by induction, we have $D_t \le D_{t_0} + (t - t_0) \frac 2\mu \ln(n)^2 \le 2 D_{t_0}$ throughout these $L$ iterations.%fuer die puristen: wir conditionieren nicht, sondern wir schneiden alle folgenden Events in dieses gute Event rein.
  
  Note that $L = O(\frac{\mu}{\log(n)})$, hence throughout this period we also have $f_{it} \ge \tfrac 13 - \tfrac 1 \mu L \ge 0.32$ (assuming $n$ to be sufficiently large) for all ${i \in [1..n]}$. By Lemma~\ref{lopt}, the probability that a fixed search point sampled in this period equals $(1, \dots, 1)$, is at least 
  \begin{align*}
  0.32^{2D_{t_0} / 0.68} &\ge 0.32^{0.2 \ln(n) / 0.68} \\
  &= \exp(0.2 \ln(n) \ln(0.32) / 0.68) \ge n^{0.2 \ln(0.32) / 0.68} \eqqcolon q.
  \end{align*}
  Since $0.2 \ln(0.32) / 0.68 > -0.34$ and $L = \Omega(\frac{\mu}{\log(n)^2})$, the probability that $(1, \dots, 1)$ is not sampled in this period is at most 
  \[(1 - q)^{2L} \le \exp(-2qL) \le \exp(- n^{0.2 \ln(0.32) / 0.68} \cdot \Omega(\tfrac{\mu}{\ln(n)^2})) \le \exp(-\Omega(n^{0.16})).\]
\end{proof}

\textbf{Sampling search points with different $1$-norm:}
To argue that the cGA makes at least some small progress, we shall use the following blunt estimate for the probability that two bit strings $x, y \sim \Sample(f)$ sampled from the same product distribution have a different distance from the all-ones string (and, by symmetry, from any other string, but this is a statement which we do not need here).

\begin{lemma}\label{ldiff}
  Let $n \in \N$, $m \in [\frac n2..n]$, and $f \in [\frac 1n,1-\frac1n]^m$. Let $x^1,x^2 \sim \Sample(f)$ be independent. Then $\Pr[\|x^1\|_1 \neq \|x^2\|_1] \ge \frac 1 {16}$.
\end{lemma}

\begin{proof}
  For all $v \in \R^m$ and $a, b \in [1..m]$ with $a \le b$ we use the abbreviation $v_{[a..b]} \coloneqq \sum_{i=a}^b v_i$. We first argue that by symmetry, we can assume that $f_{[1..m]} \le \frac m2$. Indeed, let $f_{[1..m]} > \frac m2$ and assume the claim shown for the case that $f_{[1..m]} \le \frac m2$. Let $\bar f = \mathbf{1}_m - f$ and $\bar x^1, \bar x^2 \sim \Sample(\bar f)$ independent. Then
  \begin{align*}
  \Pr[\|x^1\|_1 = \|x^2\|_1] 
  & = \sum_{i = 0}^m \Pr[\|x^1\|_1 = i = \|x^2\|_1] \\
  & = \sum_{i = 0}^m \Pr[\|x^1\|_1 = i] \cdot \Pr[\|x^2\|_1 = i] \\
  & = \sum_{i = 0}^m \Pr[\|\bar x^1\|_1 = m - i] \cdot \Pr[\|\bar x^2\|_1 = m - i] \\
  & = \sum_{i = 0}^m \Pr[\|\bar x^1\|_1 = m - i = \|\bar x^2\|_1] \\
  & = \Pr[\|\bar x^1\|_1 = \|\bar x^2\|_1] \le \tfrac{15}{16},
  \end{align*}
  where the last estimate follows from our assumption and the fact that $\bar f_{[1..m]} \le \frac m2$. This shows the claim for $f$ and justifies that in the remainder, we assume $f_{[1..m]} \le \frac m2$.
  
  Without loss of generality, we may further assume that $f_i \le f_{i+1}$ for all $i \in [1..m-1]$. We have $f_{\lfloor m/4 \rfloor} \le \frac 23$ as otherwise 
  \[f_{[1..m]} \ge f_{[\lfloor m/4 \rfloor+1..m]} > \tfrac 23  (m - \lfloor \tfrac m4 \rfloor) \ge \tfrac 23 \cdot \tfrac 34 m = \tfrac m2,\]
  contradicting our assumption. 
  
  Let $\ell$ be minimal such that $f_{[1..\ell]} \ge \frac 18$. Since $\ell \le \frac n8 \le \frac m4$, we have $f_\ell \le \frac 23$ and thus $f_{[1..\ell]} \le \frac 18 + \frac 23 = \frac{19}{24}$.
  
  For $j \in \{0,1\}$ let $q_j = \Pr[x^1_{[1..\ell]}=j] = \Pr[x^2_{[1..\ell]}=j]$. We compute 
  \begin{align*}
  \Pr[\|x^1\|_1 &\neq \|x^2\|_1] \\
  & \ge \Pr[x^1_{[1..\ell]} = x^2_{[1..\ell]} \wedge x^1_{[\ell+1..n]} \neq x^2_{[\ell+1..n]}]\\
  & \quad + \Pr[x^1_{[1..\ell]} \neq x^2_{[1..\ell]} \wedge x^1_{[\ell+1..n]} = x^2_{[\ell+1..n]}]\\
%  & \ge \Pr[x^1_{[1..\ell]} = x^2_{[1..\ell]} \mid x^1_{[\ell+1..n]} \neq x^2_{[\ell+1..n]}] \Pr[x^1_{[\ell+1..n]} \neq x^2_{[\ell+1..n]}]\\
%  & \quad + \Pr[x^1_{[1..\ell]} \neq x^2_{[1..\ell]} \mid x^1_{[\ell+1..n]} = x^2_{[\ell+1..n]}] \Pr[x^1_{[\ell+1..n]} = x^2_{[\ell+1..n]}]\\
  & = \Pr[x^1_{[1..\ell]} = x^2_{[1..\ell]}] \Pr[x^1_{[\ell+1..n]} \neq x^2_{[\ell+1..n]}]\\
  & \quad + \Pr[x^1_{[1..\ell]} \neq x^2_{[1..\ell]}] \Pr[x^1_{[\ell+1..n]} = x^2_{[\ell+1..n]}]\\
  & \ge  \min\{\Pr[x^1_{[1..\ell]} = x^2_{[1..\ell]}],\Pr[x^1_{[1..\ell]} \neq x^2_{[1..\ell]}]\} \Pr[x^1_{[\ell+1..n]} \neq x^2_{[\ell+1..n]}]\\
  & \quad +  \min\{\Pr[x^1_{[1..\ell]} = x^2_{[1..\ell]}],\Pr[x^1_{[1..\ell]} \neq x^2_{[1..\ell]}]\} \Pr[x^1_{[\ell+1..n]} = x^2_{[\ell+1..n]}]\\
  & \ge \min\{\Pr[x^1_{[1..\ell]} = x^2_{[1..\ell]}],\Pr[x^1_{[1..\ell]} \neq x^2_{[1..\ell]}]\}\\
  & \ge \min\{q_0^2 + q_1^2, 2q_0 q_1\} = 2 q_0 q_1,
%
%  & \ge \min\{\Pr[x^1_{[1..\ell]} = x^2_{[1..\ell]} = 0],\Pr[\{x^1_{[1..\ell]},x^2_{[1..\ell]}\} = \{0,1\}]\}\\
%  & \ge \min\{(\tfrac 16)^2, 2 \cdot \tfrac 16 \cdot \tfrac 5{36}\} = \tfrac 5 {108}.
  \end{align*}
  the latter by the inequality of the arithmetic and geometric mean. Using Bernoulli's inequality, we estimate coarsely   
  \begin{align*}
  q_0 &= \prod_{i=1}^\ell (1-f_i) \ge 1 - f_{[1..\ell]}, \\
  q_1 &= \sum_{i=1}^\ell f_i \prod_{j \in [1..\ell] \setminus \{i\}} (1-f_i) \ge f_{[1..\ell]} (1 - f_{[1..\ell]}).
  \end{align*}
  Since the function $z \mapsto z(1-z)^2$ is unimodal in $[0,1]$, the minimum in any subinterval of $[0,1]$ is necessarily found at a boundary of the interval. We thus obtain 
  \begin{align*}
  2 q_0 q_1 &\ge 2 \min\{z(1-z)^2 \mid z \in [\tfrac 18, \tfrac{19}{24}]\} \\
  &= 2 \min\{z(1-z)^2 \mid z \in \{\tfrac 18, \tfrac{19}{24}\}\} = 2 \tfrac {19}{24} (\tfrac 5{25})^2 \ge \tfrac 1 {16}.
  \end{align*}
\end{proof}

%*********************************
\section{An Upper Bound for the Runtime of the cGA on Jump Functions}
%*********************************

In this section, we state precisely and prove our $O(\mu \sqrt n)$ upper bound for the runtime of the cGA on jump functions with small jump size~$k$. With the smallest admissable hypothetical populations size $\mu = \Theta(\sqrt n \log n)$, it gives a runtime guarantee of $O(n \log n)$. Our result includes the trivial case $k=1$, that is, the \onemax benchmark function, a result that was known before.

Our results are true not only for jump functions, but for the larger class of all functions that (apart from a uniform additive constant) agree with \onemax on $\{0,1\}^n \setminus G_{nk}$ and that have $(1,\ldots,1)$ as an optimum (recall that $G_{nk}$ was defined to be the gap region of $\jump_{nk}$, see~\eqref{eq:gnk}). This observation is interesting in its own right, e.g., it yields that our result also holds for the plateau functions defined in~\cite{AntipovD18}. It also helps us formulating the proofs in a more concise manner since we can now assume that $k$ is sufficiently large (since a jump function with small jump parameter $k'$ is included in this larger class for all $k \ge k'$). 

To make things precise, for all $n \in \N$ and $k \ge 1$ we define the class of \emph{subjump} functions with jump size $k$ as the class of all functions $\calF : \{0,1\}^n \to \R$ such that there is a $K \in \N$ such that
\begin{itemize}
\item $\calF(x) = \|x\|_1+K$ if $\|x\|_1 \in [0..n-k] \cup \{n\}$,
\item $\calF(x) \le  n+K$ if $\|x\|_1 \in [n-k+1..n-1]$
\end{itemize}
for all $x \in \{0,1\}^n$. Here the prefix \emph{sub} is to be understood in the sense that these functions are seen as at most as hard as the true jump function with jump size $k$ or that, if $\calF$ is a jump function, its jump size is at most $k$. It is not to be understood in the sense that a subjump functions is pointwise less or equal to the corresponding jump function (which is not true). 

As said before, subjump functions have an optimum at $(1,\dots,1)$, but there could be others in the gap region $G_{nk}$. We continue to call $G_{nk}$ the \emph{gap} even though this is not for all subjump functions a fitness valley. We note that the class of subjump functions with jump size $k$ includes all subjump functions with jump size smaller than $k$, in particular, all jump functions with smaller jump size and thus also the \onemax function (to ensure this property, we needed the additional variable $K$ in the definition). 

Without further details, we note that many previously proven upper bounds for runtimes on jump functions are also valid for subjump functions. However, this might not be true for results that heavily exploit the particular structure of the true jump functions, say the symmetry of the set of local optima, such as the results on crossover-based algorithms.

The main result of this section is the following runtime guarantee for subjump functions.

\begin{theorem}\label{thm:upper}
  Let $k \le \frac 1 {20} \ln(n)-1$. For a sufficiently large constant $c_\mu$, let $\mu \ge c_\mu \sqrt n \ln(n)$, but polynomially bounded in $n$. Then the cGA with frequency boundaries (Algorithm~\ref{alg:cga}) with hypothetical population size $\mu$ with probability $1 - o(1)$ finds the optimum of any $n$--dimensional subjump function with jump size $k$ in time $O(\mu \sqrt n)$. This time is $O(n \log n)$ when $\mu = \Theta(\sqrt n \ln(n))$.
\end{theorem}

We start by giving a rough overview of the proof in the following subsection and then state the formal proof in the two subsequent subsections.

\subsection{Proof Overview}

We now give a brief overview of our runtime analysis and show how the different partial results work together. We leave it to the reader to read this section now or after the presentation of the partial results (or twice).

In our analysis, we roughly distinguish three phases of the optimization process. The first phase, analyzed in Lemma~\ref{ljump}, lasts until for the first time the frequency distance $D_t \coloneqq n - \|f_t\|_1$ is $O(\log n)$ with a large implicit constant. During this phase, by Lemma~\ref{lsample} and a union bound, with high probability we will never sample a solution in the gap. Consequently, we can pretend that we are optimizing the \onemax function and use our analysis of Lemma~\ref{lonemax}, which reuses arguments of the classic result by Droste~\cite{Droste06} including Lemma~\ref{ldroste}. 

The second phase, analyzed in Lemma~\ref{lmiddle}, then lasts until we have a $D_t$ value of less than $2k$ (or less than some constant in the case of a very small $k$). In this phase, we use the drift computed in Lemma~\ref{ldrift}. We profit from the fact that in this phase we only need to obtain a moderate decrease of $D_t$ and apply the additive drift theorem (Theorem~\ref{tdrift}\ref{it:driftub}) with the smallest drift that can occur in this phase, which is $\Omega(\frac {1}{\mu})$. Since this phase is so short, a simple Markov bound suffices to show that the phase ends with high probability in due time. 

Once we have reached a $D_t$ value of $O(k)$, we have a reasonable chance to sample the optimum as shown in Lemma~\ref{lfinish1}. Since in this third phase samples in the gap occur frequently, we have less control over $D_t$, in particular, we cannot exhibit an expected decrease of~$D_t$. We therefore pessimistically estimate $D_t$ as if $D_t$ would always increase, which gives (apart from the boundary effects described in Lemma~\ref{lboundary}) an increase of $| \|x^1\|_1 - \|x^2\|_1 |$. Since $D_t$ is small, these increases are small as well, as again ensured by Lemma~\ref{lsample}. With this observation, we can argue that we have a $D_t$ value of $O(k)$ for almost $\mu$ iterations, which together with Lemma~\ref{lopt} shows that we sample the optimum with high probability.

All the arguments above need that the frequencies are bounded away from the lower boundary of $\frac 1n$, more precisely, that they are $\Omega(1)$ at all times. In the first two phases, we ensure this via Lemma~\ref{lconc}, our general result for random processes that are not Markov processes. To this aim, we estimate the probabilities of certain frequency changes by adjusting this data from the \onemax process (Lemma~\ref{lonemax2}, taken from  Sudholt and Witt~\cite{SudholtW19}) via a pessimistic estimate of the negative influence of search points sampled in the gap. For the third phase, the fact that this phase only last $o(\mu)$ iterations implies that frequencies change by at most $o(1)$, hence the $\Omega(1)$ lower bound remains intact.

\subsection{Proof Ingredients}

In this section, we prove separately the main arguments needed in our final proof. We also state some known results on how the cGA optimizes \onemax.

The following result is a weaker form of what was shown in the proof of Lemma~5 in~\cite{Droste06}. The result of  Lemma~5 in~\cite{Droste06}, bounding the expected progress instead of showing that a certain progress can be observed with constant probability, is not sufficient for our purposes, see the discussion below. 
%note: geht auch mit was anderem als $1/3$, dann wird das $1/5$ halt schwaecher.

\begin{lemma}[\cite{Droste06}]\label{ldroste}
  There is a constant $C > 0$ such that the following holds. Let $n \in \N$ and $D \in \N$. Let $f \in [\frac 13,1]^n$ such that $\|f\|_1 \le n - D$. Let $x^1, x^2 \sim \Sample(f)$ independent. Then 
  \[\Pr\left[\left|\|x^1\|_1 - \|x^2\|_1\right| \ge \tfrac 15 \sqrt{D}\right] \ge C.\]%\qed
\end{lemma}

We use this lemma to now conduct a (partial) runtime analysis of the cGA on \onemax. Such an analysis is helpful for our purposes since the optimization process of the cGA on a subjump function is identical to the one on the \onemax function as long as no search point in the gap region is sampled. 

Our analysis on \onemax differs from Droste's analysis of the cGA on \onemax~\cite[Theorem~8]{Droste06} in several respects. First, we aim at a guarantee that holds with high probability. For this reason, we cannot use the approach via additive drift, and this is the reason why we need Lemma~\ref{ldroste} instead of Lemma~5 from~\cite{Droste06}. 

We note that Droste's drift argument is also not perfectly complete. In each of his $\Theta(n)$ relatively short phases, he uses additive drift to estimate from the expected progress the time to reach the phase target, but he ignores the fact that his expected progress also takes into account progress beyond the phase target. This could lead to an overestimation of the progress. This problem does not occur in the additive drift theorem as stated in Theorem~\ref{tdrift}, since there the process lives in the non-negative numbers and the process target is zero. We have no doubt, though, that this technical gap can be fixed with additional arguments.

We regard the cGA with non-trivial boundaries, which requires additional arguments as the capping of the frequencies can change the drift of the frequency sum (albeit by not a lot, as our proof shows). We note that without these extra arguments, our proof also applies to the setting without boundaries. 

We only regard the time needed to reach a frequency vector with constant distance to the all-ones vector. We note that our analysis can be extended to also give a bound for the time to sample an optimal solution, but we do not need such a result (and in fact, such a result is implied by our main result). Also, a simplified version of our proof would apply to the cGA without boundaries.

\begin{lemma}\label{lonemax}
  Let $C$ be the constant from Lemma~\ref{ldroste}. Consider a run of the cGA with $\mu \ge \log_2 n$ on the \onemax benchmark function. Let $D_t \coloneqq n - \|f_t\|_1$ for all $t$. Let $K$ be a sufficiently large constant. Let $T$ be the first time that $D_t \le K$ or that there is an $i \in [1..n]$ with $f_{it} < \frac 13$. Then
  \[\Pr\left[T \ge \frac{10(2+\sqrt 2)}{C}\,\mu \sqrt n \right] = \exp(-\Omega(\mu)).\]
\end{lemma}

We formulated the result above in the slightly cumbersome manner of giving a time guarantee for the event of reaching a near-optimal frequency vector or reaching a frequency below $\frac 13$. By Lemma~\ref{lconc} we will be able to rule out the latter event via a simple union bound over the failure probabilities. This approach is technically simpler than conditioning on the frequencies to not go below $\tfrac 13$ and then working in the conditional probability space.

\begin{proof}[Proof of Lemma~\ref{lonemax}]
Define $D'_t \coloneqq n - \|f'_t\|_1$ for all $t \ge 1$.
For $i = 1, 2, \dots$ let $d_i = 2^{-i} n$. Without loss of generality, we may assume that $K = 2^{-\ell-1}n$ for some $\ell \in \N$. Note that $\ell \le \log_2 n$. We say that the optimization process enters Phase $i$ (and thus leaves its current phase) when for the first time $D_t \le d_i$. Note that we stay in Phase $i$ even when after entering this phase $D_t$ increases beyond $d_i$. Note further that, by definition, the process starts in Phase~1. We also say that the current phase ends when a frequency reaches a value below $\frac 13$.

We analyze the time spend in each Phase $i \le \ell$ (when assuming that all frequencies are at least $\frac 13$ at the start of the phase) and show that this time, with probability at least $1 - \exp(-\Omega(\mu))$, is at most $T_i = \lceil 20 \frac 1C \mu \sqrt{d_{i+1}} \rceil$. Let $t'$ be the iteration in which the process enters Phase $i$. To ease the argument, we now consider exactly $T_i$ iterations. In case the phase ends earlier, we shall from that point on regard an artificial process, with a slight abuse of notation also denoted by $D_t$ and $D'_t$, that satisfies the conditions 
\begin{align*}
&\Pr[D'_{t+1} = D_t - \tfrac 1{5 \mu} \sqrt{d_{i+1}} \mid D_t] = C,\\
&\Pr[D'_{t+1} = D_t \mid D_t] = 1-C,\\
&\Pr[D_{t+1} = D'_{t+1} \mid D'_{t+1}] = 1.
\end{align*}
Such an artificial extension of a process was, to the best of our knowledge, in the theory of evolutionary algorithms first used in~\cite{DoerrHK11}.

When all frequencies are at least $\frac 13$, by Lemma~\ref{ldroste} we have $\Pr[|\|x^1\|_1 - \|x^2\|_1| \ge \tfrac 15 \sqrt{D_t}] \ge C$. Since we have $\|y^1\|_1 \ge \|y^2\|_1$ when optimizing \onemax, we have that $D'_{t+1}$ with probability at least $C$ satisfies $D'_{t+1} \le D_t - \tfrac 1{5\mu} \sqrt{D_t} \le D_t - \tfrac 1{5\mu} \sqrt{d_{i+1}}$. We call this a \emph{success}. Note that the probability for a success is at least $C$ regardless of what happened before in this phase. Consequently, in $T_i$ iterations, we not only have an expected number of at least $20 \mu \sqrt{d_{i+1}}$ successes, but, using the multiplicative Chernoff bounds (Theorem~\ref{tchernoff}) and the fact that ``sequential independence'' suffices for Chernoff bounds to be admissible (Lemma~11 in~\cite{DoerrJ10} or Section~1.10.2.1 in~\cite{Doerr20bookchapter}), we also have at least $10 \mu \sqrt{d_{i+1}}$ successes with probability at least $1 - \exp(-\tfrac 52 \mu \sqrt{d_{i+1}})$. Note that with probability one we have $D'_{t+1} \le D_t$, again because $\|y^1\|_1 \ge \|y^2\|_1$.

By Lemma~\ref{lboundary}~\ref{it:boundaryU}, we have $D_{t+1} \preceq D'_{t+1} + \frac 1 \mu \Bin(n, \frac 2n)$, again regardless of what happened in earlier iterations. Consequently, the total number of times we increase $D_t$ by $\frac 1\mu$ due to reaching an upper frequency boundary can be estimated by a sum of $T_i n$ independent binary random variables with success probability $\frac 2n$. Hence the expectation of this number is at most $2 T_i \le 40 \frac 1C \mu \sqrt{d_{i+1}} + 2$ and, by Theorem~\ref{tchernoff}, with probability at least $1 - \exp(\frac{2T_i}{3}) \ge 1 - \exp(-\frac{40}{3} \frac 1C \mu \sqrt{d_{i+1}})$ this number is at most $4 T_i = 80 \frac 1C \mu \sqrt{d_{i+1}}+4$.

Taking these two observations together, we see that with probability \[1 - \exp\left(-\frac 52 \mu \sqrt{d_{i+1}}\right)  - \exp\left(-\frac{40}{3} \frac 1C \mu \sqrt{d_{i+1}}\right)  = 1 - \exp\left(-\Omega(\mu)\right),\] we have 
\begin{align*}
D_{t'+T_i} &\le D_{t'} - 10 \mu \sqrt{d_{i+1}} \cdot \tfrac 1{5\mu} \sqrt{d_{i+1}} + \tfrac 1\mu (80 \tfrac 1C \mu \sqrt{d_{i+1}}+4) \\
&= D_{t'} - 2 d_{i+1} + \tfrac{80}{C} \sqrt{d_{i+1}} + \tfrac4\mu.
\end{align*} 
Since $K = 2^{-\ell-1} n \le d_{i+1}$ was chosen sufficiently large, we can assume that $- 2 d_{i+1} + \frac{80}{C} \sqrt{d_{i+1}} + \frac{4}{\mu} \le - d_{i+1}$ and thus $D_{t'+T_i} \le D_{t'} - d_{i+1}$, that is, $D_{t'+T_i}$ belongs to a later phase already. Consequently, we have that with probability at least $1 - \exp(-\Omega(\mu))$, at most $T_i$ rounds are spend in Phase~$i$. 

We finally show our claim first by noting that there are only $O(\log n)$ phases, hence with probability at least $1 - O(\log n) \exp(-\Omega(\mu)) = 1 - \exp(-\Omega(\mu))$ no phase takes longer than the desired $T_i$ iterations, and second by computing 
\begin{align*}
\sum_{i=1}^{\ell} T_i &\le \ell + \sum_{i=1}^{\ell} 20 \frac 1C \mu \sqrt{2^{-(i+1)}n} \le \frac{10}{C}\mu \sqrt n \sum_{i=0}^\infty (2^{-1/2})^i \\
&= \frac{10}{C}\mu \sqrt n \frac{1}{1-2^{-1/2}} = \frac{10(2+\sqrt 2)}{C}\mu \sqrt n.
\end{align*}
\end{proof}

Lemma~\ref{lonemax} can be extended to give a time bound for subjump function as long as the target distance from the optimum is sufficiently large. 

\begin{lemma}\label{ljump}
  Let $C$ be the constant from Lemma~\ref{ldroste} and let $C_\mu$ be any constant. Consider a run of the cGA with $\mu = \omega(\log n)$ and $\mu \le n^{C_\mu}$ on a subjump function $\calF$ with jump size $k \le \tfrac 1 {20} \ln n$. Let $D_t \coloneqq n - \|f_t\|_1$ for all $t$. Let $K = (8C_\mu+12)\ln n$. Then with probability $1 - O(\frac 1n)$, there is a $t \le T \coloneqq \frac{10(2+\sqrt 2)}{C}\,\mu \sqrt n$ such that $D_t \le K$ or $f_{it} < \tfrac 13$ for some $i \in [1..n]$.
\end{lemma}

\begin{proof}
  We regard the modified optimization process where we start with a run of the cGA on $\calF$, but change the fitness function to \onemax when for the first time $D_t \le K$. Clearly, this modified process satisfies our claim if and only if the original process on $\calF$ does. 
  
  We now couple the modified process to the optimization process of the cGA with same $\mu$ value on the \onemax function. We construct this coupling as follows. For each $t = 1, 2, \ldots$ and each $j \in [1..2]$ we let $r^{tj} \in [0,1]^n$ be a vector chosen uniformly at random. If $f_t$ is the frequency vector of the modified or the \onemax process, then the $j$-th sample $x^{tj} \in \{0,1\}^n$ in iteration $t$ of this process is defined by $x^{tj}_i = 1$ if and only if $r^{tj}_i \le f_{it}$. Clearly, the two (marginal) processes defined this way are identically distributed to the two processes we wanted to couple. More interestingly, the two processes in the coupling are identical up to the point where the modified subjump process samples a search point in the gap region and thus before it changed the fitness to \onemax. If we denote the probability of this event happening within the first $T$ iterations by $p$, then by Lemma~\ref{lonemax} and a union bound over the two failure probabilities, we have that with probability at least $1 - (\exp(-\Omega(\mu)) + p)$, within $T$ iterations the modified process has reached a $D_t$ value of at most $K$ or has reached a frequency below $\frac 13$. 
  
  Hence it remains to show that $p$ is sufficiently small. For this we note that by Lemma~\ref{lsample}, the probability that in the modified subjump process before the switch to the \onemax fitness a particular search point $x^{tj}$ lies in the gap, is at most 
  \[\Pr[d(x^{tj}) \le k] \le \Pr[d(x^{tj}) \le \tfrac 12 D_t] \le \exp(-\tfrac 18 D_t) \le \exp(-\tfrac 18 K) \le n^{-C_\mu} n^{-1.5},\] where we wrote $d(x^{tj}) := n - \|x^{tj}\|$ as earlier in this work. By a union bound, $p \le 2T n^{-C_\mu} n^{-1.5} = O(\frac{1}{n})$.
\end{proof}

We now analyze the drift in $D_t$ when we are that close to the gap that we cannot assume anymore that we never sample a search point in the gap. We recall the definition of the gap by 
\[G \coloneqq G_{nk} \coloneqq \{x \in \{0,1\}^n \mid n-k < \|x\|_1 < n\}\] 
and we further define $G^+ \coloneqq G \cup \{(1,\dots,1)\}$.

A difficulty here, which was not treated fully rigorously in~\cite[Lemma~3.1]{HasenohrlS18}, is that the event $G_t$ that $x^1$ or $x^2$ lie in the gap and the random variable $|\|x^1\|_1 - \|x^2\|_1|$ are not independent. Consequently, the estimate $E[D_t - D_{t+1} \mid D_t] =  \frac 1\mu |\|x^1\|_1 - \|x^2\|_1| (1 - 2\Pr[G_t])$ is not correct. In fact, the correlation is indeed not in our favor. When $|\|x^1\|_1 - \|x^2\|_1|$ is large, the probability that a search point in the gap was sampled (and thus the frequency update is done in the unwanted direction) is higher. We solve this difficulty by computing an estimate for $|\|x^1\|_1 - \|x^2\|_1|$ conditional on that at least one of the search points lies in the gap. 

\begin{lemma}\label{ldrift}
  Let $\mu$ be arbitrary (but, as always in this work, satisfying the well-behaved frequency assumption). Let $k \in [1..\frac 12 n - 1]$. Consider an iteration $t$ of the cGA optimizing a subjump function with jump size $k$ started with a frequency vector $f_t$ such that $D_t \coloneqq n - \|f_{t}\|_1 \ge 2k$ and $f_t \in [\frac{1}{3}, 1 - \frac{1}{n}]^n$. Then 
  \[E[\mu D_{t} - \mu D_{t+1}]  \ge \tfrac{1}{5} C \sqrt{D_t} - 6 D_t \exp(-\tfrac 18 D_t) -2,\]
  where $C$ is the constant from Lemma~\ref{ldroste}.
\end{lemma}

\begin{proof}
From the definition of the cGA, we note that when $x^1$ and $x^2$ are both not in $G^+$, then $D'_{t+1} \coloneqq n - \|f'_{t+1}\|_1$ satisfies $D'_{t+1} = D_t - \frac 1\mu |\|x^1\|_1 - \|x^2\|_1|$ as if we were optimizing \onemax. In all other cases, we have $D'_{t+1} \le D_t + \frac 1\mu |\|x^1\|_1 - \|x^2\|_1|$. Consequently, 
\begin{align*}
&E[\mu D_{t} - \mu D'_{t+1}]\\ 
&\ge \Pr[x^1, x^2 \notin G^+] \, E[|\|x^1\|_1 - \|x^2\|_1| \mid x^1, x^2 \notin G^+] \\
& \quad - \Pr[\{x^1, x^2\} \cap G^+ \neq \emptyset] \, E[|\|x^1\|_1 - \|x^2\|_1| \mid \{x^1, x^2\} \cap G^+ \neq \emptyset]\\
&= E[|\|x^1\|_1 - \|x^2\|_1|] \\
& \quad - 2 \Pr[\{x^1, x^2\} \cap G^+ \neq \emptyset] \, E[|\|x^1\|_1 - \|x^2\|_1| \mid \{x^1, x^2\} \cap G^+ \neq \emptyset].
\end{align*}
When the frequencies are all at least $\frac 13$, we conclude from Lemma~\ref{ldroste} that $E[|\|x^1\|_1 - \|x^2\|_1|] \ge \frac{1}{5} C \sqrt{D_t}$. 

For the contribution when search points are in~$G^+$, we first note that the second bound of Lemma~\ref{lsample} (with $\delta = \frac 12$ and $D^- = D_t$) and $D_t \ge 2k$ yield 
\begin{equation*}
\Pr[x^1 \in G^+] \le \Pr[d(x^1) \le \tfrac 12 D_t] \le \exp(-\tfrac 18 D_t).
\end{equation*}
Then, exploiting the symmetry between $x^1$ and $x^2$, counting the case $x^1, x^2 \in G^+$ twice, and using again $\frac 12 D_t \ge k$, we compute
\begin{align*}
\Pr[&\{x^1, x^2\} \cap G^+ \neq \emptyset] \, E[|\|x^1\|_1 - \|x^2\|_1| \mid \{x^1, x^2\} \cap G^+ \neq \emptyset] \\
& \le 2 \Pr[x^1 \in G^+] \, E[|\|x^1\|_1 - \|x^2\|_1| \mid x^1 \in G^+] \\
& \le 2 \Pr[x^1 \in G^+] \, \left(E[|\|x^1\|_1 - n| \mid x^1 \in G^+] + E[|n - \|x^2\|_1|]\right) \\
& \le 2 \Pr[x^1 \in G^+] \, \left(k + D_t\right) \\
& \le 2 \exp(-\tfrac 18 D_t) (\tfrac 12 D_t + D_t) = 3 \exp(-\tfrac 18 D_t) D_t.
\end{align*}

In summary, we have
\[E[\mu D_{t} - \mu D'_{t+1}] \ge \tfrac{1}{5} C \sqrt{D_t} - 6 D_t \exp(-\tfrac 18 D_t).\]

By Lemma~\ref{lboundary}, we further have $E[\mu D_{t+1} - \mu D'_{t+1}] \le 2$. Consequently, recalling that the linearity of expectation holds also for dependent random variables, we have 
\begin{align*}
E[\mu D_{t} - \mu D_{t+1}] &= E[\mu D_{t} - \mu D'_{t+1}] - E[\mu D_{t+1} - \mu D'_{t+1}] \\
&\ge \tfrac{1}{5} C \sqrt{D_t} - 6 D_t \exp(-\tfrac 18 D_t) -2.
\end{align*}
\end{proof}

From Lemma~\ref{ldrift}, we obtain the following coarse estimate for the time to reach a frequency distance $D_t$ below $2k$ (or at least below some constant). 

\begin{lemma}\label{lmiddle}
  Let $k \in [1..\frac n2 -1]$. Consider a run of the cGA with arbitrary hypothetical population size $\mu$ (satisfying the well-behaved frequency assumption) and started with a fixed frequency vector $f_0$ instead of the usual initialization $f_0 = (\frac 12, \ldots, \frac 12)$. For all $t \ge 0$, let $D_t := n - \|f_t\|_1$. Let $D'' \ge 2k$ and at least some sufficiently large constant (depending on the constant $C$ from Lemma~\ref{ldroste}). Let $T$ be the first time $t$ that this run reaches a frequency vector $f_t$ with $D_t < D''$ or that there is a frequency $f_{it}$ that is less than $\frac 13$. Then 
  \[E[T] \le \mu D_0.\]
\end{lemma}

\begin{proof}
  Based on the run of the cGA, we define a random process $\tilde D_t$ as follows. If for some $s \in [0..t]$ we have $D_s < D''$ or there is an $i \in [1..n]$ with $f_{is} < \frac 13$, then $\tilde D_{t} = 0$. Otherwise, we let $\tilde D_{t} = D_t$. In other words, the process $(\tilde D_t)$ agrees with $(D_t)$ while $(D_t)$ is at least $D''$ and there is no frequency below $\frac 13$, and then is constant zero. 

  By Lemma~\ref{ldrift} and using our assumption that $D''$ is a large absolute constant, we have $E[\tilde D_t - \tilde D_{t+1}] \ge \frac{1}{\mu}$ whenever $D_t \ge D''$ and $f_t \in [\frac 13,1]^n$, that is, we have $E[\tilde D_t - \tilde D_{t+1} \mid \tilde D_t > 0] \ge \frac{1}{\mu}$ for all $t \ge 0$.

  Since $T = \inf\{t \ge 0 \mid \tilde D_t = 0\}$, the additive drift theorem (Theorem~\ref{tdrift}\ref{it:driftub}) yields $E[T] \le \frac{D_0}{1/\mu}$. 
\end{proof}

We end this section of preliminary results with an argument showing that the frequencies stay away from the lower boundary for a decent amount of time. On the formal level, this argument will be used to argue that at the times $T$ estimated on Lemmas~\ref{ljump} and~\ref{lmiddle}, we have the desired small $D_t$ value and not the case that a frequency went below $\frac 13$. On the intuitive level, this argument is necessary for two reasons. On the one hand, as can be seen from the proof or via simple counter-examples, a lower bound for the probability of sampling the optimum such as Lemma~\ref{lopt} is not anymore true if arbitrarily small frequencies are allowed. On the other hand, small frequencies support that the two offspring sampled in one iteration agree in the corresponding bit. In this case, no change of the frequency is possible, which slows down the progress and rules out progress guarantees such as Lemma~\ref{ldrift}. 

A guarantee that all frequencies stay away from the lower boundary in a run of the cGA on jump functions was also given in~\cite[Lemma~2.4]{HasenohrlS18}. Unfortunately, the proof appears not complete to us. It seems to us that the main technical prerequisite of this result, Lemma~2.2~in~\cite{HasenohrlS18} with a proof of a little over one page in the condensed proceedings style, is not correct for two reasons. Since the proof of Lemma~2.2 never refers to the frequency boundaries, it is not clear if it is applicable for the cGA with these boundaries. Rather, a frequency vector having one entry $f_{it}=\frac 1n$ and another one $f_{jt}=1 - \frac 1n$ seems to be a counter-example (note that the frequency vector is called $p_t$ instead of $f_t$ in~\cite{HasenohrlS18}). However, also for the case without boundaries counter-examples seem to exist for all values of $\mu$, e.g., the frequency vector $f_t = (\frac 1 {100}, \frac 12)$. 

We did not see how to repair the otherwise elegant argument via the Azuma-Hoeffding inequality. For this reason, using a sequence of elementary reductions, we argue that the true random process of a frequency, which is not a Markov process when regarding one frequency in isolation, can be pessimistically replaced by a fair random walk on an unbounded frequency domain. For the analysis of the latter, classic Chernoff bounds can be used. This general approach was also taken in~\cite{Droste06}, however in the easier situation that there are no frequency boundaries (apart from the trivial boundaries, which are absorbing). For this reason, some additional arguments are necessary in our situation.%\merk{what else is known in this respect? SW19 oder andere Helden brauchen sowas doch sicherlich auch? Droste machts so aehnlich wie wir, ignoriert aber boundary effect (up at least as likely as down), spart sich viele Details, und schafft nur $n^{1/2+\eps}$}

\begin{lemma}\label{lconc}
  Let $\mu$ be arbitrary (except that it satisfies the well-behaved frequency assumption). Let $\eps > 0 $. Let $Z_0, Z_1, \dots$ be any random process on $F_\mu$ (defined in~\eqref{eq:fmu}) such that 
  \begin{enumerate}
  \item $Z_0 = \frac 12$, 
  \item for all $t = 0, 1, \dots$ such that $Z_t \ge \frac 12 - \eps$ there are numbers $p_t, q_t, r_t \in [0,1]$, depending on $Z_0, Z_1, \dots, Z_t$, such that $p_t + q_t + r_t = 1$ and
  \begin{align*}
  \Pr[Z_{t+1} = Z_t \mid Z_0, \dots, Z_t] &=p_t,\\
  \Pr[Z_{t+1} = Z_t + \tfrac 1 \mu \mid Z_0, \dots, Z_t] &= q_t,\\
  \Pr[Z_{t+1} = Z_t - \tfrac 1 \mu \mid Z_0, \dots, Z_t] &= r_t.
  \end{align*}
  We further assume that %$r_t = 0$ when $Z_t = \frac 1n$, that $q_t = 0$ if $Z_t = 1 - \frac 1n$, and that 
  $q_t \ge r_t$ when $Z_t \neq 1 - \frac 1n$. 
  \end{enumerate} 
  Then for all $T \in \N$, 
  \[\Pr[\exists t \in [0..T] : Z_t < \tfrac 12 - \eps] \le 2 \exp\left(-\frac{\mu^2 \eps^2}{2T}\right).\]
\end{lemma}

\begin{proof}
  For the ease of the argument, we can without loss of generality assume that condition~(ii) also holds when $Z_t < \frac 12 - \eps$. We conduct a sequence of reductions to a fair unbiased random walk on an infinite line. We first observe that we can assume $p_t = 0$ for all $t$. The event $Z_{t+1} = Z_t$ that the process does not move only slows down the process in the sense that it visits fewer states, and thus is less likely to approach the lower boundary. 
  
  We now argue that w.l.o.g. we can assume that $q_t = r_t= \tfrac 12$ for all $t \in [0..T-1]$ except in the cases $Z_t \in \{\frac 1n, 1-\frac 1n\}$. To make this argument formal, we inductively modify $q_t$ and $r_t$ in the time interval $t \in [0..T-1]$. The modified process will be denoted by $(\tilde Z_t)_{t = 0, \ldots, T}$ and described via $\tilde q_t$ and $\tilde r_t$, $t \in [0..T-1]$, which again are functions of $\tilde Z_0, \ldots, \tilde Z_t$. We start with $(\tilde Z_t)$ being a copy of $(Z_t)$. We denote by
  \begin{align*}
  P_{it} &\coloneqq \Pr[\exists s \in [t..T] : Z_s < \tfrac 12 - \eps \mid Z_t = i],\\ 
  \tilde P_{it} &\coloneqq \Pr[\exists s \in [t..T] : \tilde Z_s < \tfrac 12 - \eps \mid \tilde Z_t = i]\\
  \end{align*}
  the ``failure probabilities'' of both processes given a particular starting point and time. 
  
  Assume that $(\tilde Z_t)$ is such that for some $t_0 \in [1..T]$ we have that for all $s \in [t_0..T-1]$ 
  \begin{enumerate}
  \item $\tilde q_s = \tilde r_s = \frac 12$ regardless of $\tilde Z_0, \ldots, \tilde Z_s$ (except in the boundary cases $\tilde Z_s \in \{\frac 1n, 1 - \frac 1n\}$);
  \item $P_{is} \le \tilde P_{is}$ for all $i \in F_\mu$.
  \end{enumerate}
  Note that our initial copy $(\tilde Z_t)$ satisfies these conditions for $t_0 = T$. We now modify $(\tilde Z_t)$ so that the new process satisfies these conditions already for $t_0-1$. To this end, let $(Z'_t, q'_t, r'_t)$ be a copy of $(\tilde Z_t, \tilde q_t, \tilde r_t)$ expect that we define $q'_{t_0-1} = r'_{t_0-1} = \frac 12$ (except in the boundary cases  $Z'_{t_0-1} \in \{\frac 1n, 1 - \frac 1n\}$). Since from time $t_0$ on $(Z'_t)$ equals $(\tilde Z_t)$, we have $P'_{is} = \tilde P_{is} \ge P_{is}$ for all $i \in F_\mu$. Since further from time $t_0$ on $(Z'_t)$ and $(\tilde Z_t)$ are a fair random walks with reflecting boundaries, a simple coupling argument shows $P'_{i-1/\mu,t_0} \ge P'_{i+1/\mu,t_0}$ for all $i \in F_{\mu} \setminus \{\frac 1n, 1 - \frac 1n\}$. From this, $\tilde r_{t_0-1} \le \tilde q_{t_0-1}$, and $\tilde r_{t_0-1} + \tilde q_{t_0-1} = 1$, we obtain for all $i \in F_\mu \setminus \{\frac 1n, 1 - \frac 1n\}$ that 
  \begin{align*}
  P'_{i,t_0-1} &= \tfrac 12 P'_{i-1/\mu, t_0} + \tfrac 12 P'_{i+1/\mu, t_0} \\
  &\ge \tilde r_{t_0-1} P'_{i-1/\mu, t_0} + \tilde q_{t_0-1} P'_{i+1/\mu, t_0} = \tilde P_{i,t_0-1} \ge P_{i,t_0}.
  \end{align*}
  In the boundary cases, we trivially have $P'_{1/n,t_0-1} = 1 =  P_{1/n,t_0-1}$ and $P'_{1 - 1/n,t_0-1} = P'_{1 - 1/n - 1/\mu, t_0} = \tilde P_{1 - 1/n - 1/\mu, t_0} \ge P_{1 - 1/n - 1/\mu, t_0} = P_{1 - 1/n, t_0-1}$. This proves our claim. An elementary induction gives a process $(\tilde Z_t)$ that satisfies (i), (ii) above from $t_0 = 0$ on. This process, then, is a simple unbiased random walk with reflecting boundaries. From (ii) we see that such an unbiased random walk is not better than the original process $(Z_t)$ in terms of avoiding to go below $\frac 12 - \eps$.
  
  Hence we can now assume that $(Z_t)$ is an unbiased random walk on $F_\mu$ with reflecting boundaries. We shall show that 
\[\Pr[\exists t \in [0..T] : Z_t \notin [\tfrac 12-\eps,\tfrac 12 +\eps]] \le 2 \exp\left(-\frac{\mu^2 \eps^2}{2T}\right).\]

 Being interested in the event that the process reaches a state outside ${[\tfrac 12-\eps,\tfrac 12 +\eps]}$ at least once, we can also drop the boundary conditions and assume that we have $Z_{t+1} \in \{Z_t - \frac 1 {\mu}, Z_t + \frac 1 \mu\}$ uniformly at random at all times~$t$. We can now rewrite the $Z_t$ as follows. Let $X_1, \dots, X_T$ be independent random variables uniformly distributed on $\{-\frac 1 \mu, \frac 1 \mu\}$. Define $Z''_t \coloneqq \frac 12 + \sum_{i=1}^t X_t$ for all $t \in [0..T]$. Then $(Z_0, \ldots, Z_T)$ and $(Z''_0, \ldots, Z''_T)$ are identically distributed. Consequently, we can apply to $(Z_t)$ and $(Z''_t)$ the additive Chernoff bound in the sharper version working also for partial sums, Theorem~\ref{tchernoffmax}, and obtain 
\begin{align*}
\Pr[\exists &t \in [0..T] : Z_t \notin [\tfrac 12-\eps,\tfrac 12 +\eps]] \\
& = \Pr[\exists t \in [0..T] : |Z_t - E[Z_t]| > \eps] \\
&\le 2 \exp\left(-\frac{2 \eps^2}{T (\frac{2}{\mu})^2}\right) = 2 \exp\left(-\frac{\mu^2 \eps^2}{2T}\right).%\qedhere
\end{align*}  
\end{proof}

To apply Lemma~\ref{lconc}, we need a deeper understanding of the random process describing a single frequency. For this, we build on the following estimate of the expected change of a frequency that is not affected by the boundaries in the \onemax process. This result was proven in~\cite[Lemma~3]{SudholtW19}. 

\begin{lemma}\label{lonemax2}
  Let $\mu$ be arbitrary (but satisfying the well-behaved frequency assumption). Consider a run of the cGA optimizing \onemax. Consider an iteration starting with a frequency vector $f_t$. Let $i \in [1..n]$ be such that $\frac 1n + \frac 1\mu \le f_{it} \le (1 - \frac 1n) - \frac 1\mu$. Then
  \[E[f_{i,t+1} - f_{it}] \ge \frac{2}{11} \frac{f_{it} (1-f_{it})}{\mu} \left(\sum_{j \neq i} f_{jt} (1-f_{jt})\right)^{-1/2}.\]
\end{lemma}

From Lemmas~\ref{lconc} and~\ref{lonemax2}, we now obtain the following lower bound guarantee for the frequencies in the optimization process on subjump functions. Regarding the restriction $k \ge 17$, we recall that a subjump function with jump size smaller than $17$ also is a subjump function with jump size $17$. The lemma thus applies also to these (in a suitable manner). We could have alternatively formulated the lemma for all $k$ and defined $D'' = \max\{2k+1,35\}$.

\begin{lemma}\label{lfreqlb}
  Let $k \in [1..n]$ be arbitrary. Consider the run of the cGA with hypothetical population size $\mu$ on a subjump function with jump size $k \ge 17$. Let $D_t = n - \|f_t\|_1$ for all $t$. Let $D'' = 2k+1$ and $ T'' = \inf\{t \ge 0 \mid D_t \le D''\}$. Then for all $T \in \N$, with $T''' \coloneqq \min\{ T'',T\}$, we have 
  \[\Pr[\exists i \in [1..n] \exists t \in [0..T'''] : f_{it} < \tfrac 13] \le 2 n \exp\left(- \frac{\mu^2}{72 T}\right).\] 
\end{lemma}  

\begin{proof}
Consider some time $t$ such that $f_t \in [\frac 13,1]^n$ and $D_t \ge D''$. Consider a fixed bit $i\in [1..n]$ such that $f_{it} \neq 1-\frac 1n$. If we were optimizing the \onemax function, then by Lemma~\ref{lonemax2},
  \begin{align*}
  \Pr[&f_{i,t+1} = f_{it} + \tfrac 1\mu] - \Pr[f_{i,t+1} = f_{it} - \tfrac 1\mu]\\
  & = \mu E[f_{i,t+1} - f_{it}]\\
  & \ge \frac{2}{11} f_{it}(1-f_{it}) \left(\sum_{j\neq i} f_{jt}(1-f_{jt})\right)^{-1/2}\\
  & \ge \frac{2}{11} f_{it}(1-f_{it}) \left(D_t\right)^{-1/2}.
  \end{align*}
  
  Regardless of whether we optimize \onemax or a subjump function, the events $f_{i,t+1} = f_{it} + \tfrac 1\mu$ and $f_{i,t+1} = f_{it} - \tfrac 1\mu$ can only occur when the two search points sampled in this iteration satisfy $x^1_i \neq x^2_i$. The definition of $f_{i,t+1}$ in the subjump case differs from the \onemax case at most when at least one of $x^1$ and $x^2$ lie in the gap $G_{nk}$. Hence the following coarse correction of the above estimate is valid for the optimization of subjump functions of jump size $k$.
  \begin{align*}
  \Pr[&f_{i,t+1} = f_{it} + \tfrac 1\mu] - \Pr[f_{i,t+1} = f_{it} - \tfrac 1\mu]\\
  & \ge \tfrac{2}{11} f_{it}(1-f_{it}) \left(D_t\right)^{-1/2} - \Pr[(x^1_i \neq x^2_i) \wedge (\{x^1,x^2\} \cap G_{nk} \neq \emptyset)].
  \end{align*}
  We now estimate this correction term. We note that 
  \begin{align*}
  \Pr[&(x^1_i \neq x^2_i) \wedge (\{x^1,x^2\} \cap G_{nk} \neq \emptyset)] \\
  &= \Pr[x^1_i \neq x^2_i] \cdot \Pr[\{x^1,x^2\} \cap G_{nk} \neq \emptyset \mid x^1_i \neq x^2_i].
  \end{align*}
  By symmetry and the union bound, we have 
  \[\Pr[\{x^1,x^2\} \cap G_{nk} \neq \emptyset \mid x^1_i \neq x^2_i] \le 2 \Pr[x^1 \in G_{nk} \mid x^1_i \neq x^2_i].\]
  Conditional on $x^1_i \neq x^2_i$, the bit string $x^1$ is sampled from $\Sample(f_t)$, however, conditional on the $i$-th bit being zero or one. In either case, to have $x^1 \in G_{nk}$, we need that $\tilde D = \sum_{j \neq i} (1-x^1_j)$ is at most $k \le \frac 12 (D_t-1)$, where we recall that $D_t \ge D'' = 2k+1$. Since $E[\tilde D] = D_t - (1-f_{it}) \ge D_t-1$, by Lemma~\ref{lsample} with $\delta = \frac 12$ this event happens with probability at most $\exp(-\tfrac 18 (D_t-1))$. Together with $\Pr[x^1_i \neq x^2_i] = 2 f_{it}(1-f_{it})$, we obtain
  \begin{align*}
  \Pr[&f_{i,t+1} = f_{it} + \tfrac 1\mu] - \Pr[f_{i,t+1} = f_{it} - \tfrac 1\mu]\\
  & \ge \tfrac{2}{11} f_{it}(1-f_{it}) \left(D_t\right)^{-1/2} - 2 f_{it}(1-f_{it}) \exp(-\tfrac 18 (D_t-1)),
  \end{align*}
  which is non-negative since $D_t \ge D'' = 2k+1 \ge 35$.
   
  Consequently, the process $(f_{it})_t$ satisfies the assumptions of Lemma~\ref{lconc} up to time $ T''$. If $ T'' < T$, we artificially extend the process (for the following argument only) by setting $f_{it} = f_{i  T''}$ for all $t \in [ T''+1..T]$. We apply Lemma~\ref{lconc} to this extended process and obtain that up to time $T$, the \mbox{$i$-th} frequency is always at least $\frac 13$ with probability $1 - 2\exp(-\frac{\mu^2}{72T})$. With a union bound over the $n$ frequencies, we have $f_t \in [\frac 13,1]^n$ up to time $T$ with probability at least $1 - 2n\exp(-\frac{\mu^2}{72T})$ in the extended process, and up to time $T'''$ in the true process.
\end{proof}

\subsection{Proof of Theorem~\ref{thm:upper}}

We are now ready to formulate the full proof of our main upper bound result.

\begin{proof}[Proof of Theorem~\ref{thm:upper}]
  To allow the reader to easily check that all implicit constants can be chosen in a way that they give the claimed result, we make these constants explicit in the following proof, but note that for most of them it just suffices to choose them sufficiently large. 

  We consider the optimization of a subjump function $\calF : \{0,1\}^n \to \R$ with jump size $k \le \frac 1 {20} \ln(n)-1$. Without loss of generality, we can assume that $k \ge 17$.\footnote{In fact, we could just assume that $k = \lfloor \frac 1 {20} \ln(n) \rfloor -1$, but we find it more insightful to present the proof in a way that the arguments are adjusted to the true value of $k$ (assuming it to be at least $17$).}
  
  Let $\mu \ge c_{\mu} \sqrt n \ln(n)$ for a constant $c_\mu$ to be defined in a moment. Assume further that for some constant $C_\mu$ we have $\mu \le n^{C_\mu}$. Without loss of generality, we assume that $C_\mu \ge 1$. 
  
  Consider a run of the cGA with hypothetical population size $\mu$ on $\calF$. Let $D_t \coloneqq n - \|f_t\|_1$ for all $t \ge 0$. 
  
  Let $D' \coloneqq C_{D'} \ln n$, where $C_{D'} \ge 8 C_\mu + 12$ is a constant. Let $T'$ be the first time that  $D_t \le D'$ or that there is a frequency $f_{it}$ that is less than $\frac 13$. By Lemma~\ref{ljump}, with probability at least $1 - O(\frac 1n)$ we have $T' \le \frac{10(2 + \sqrt 2)}{C} \mu \sqrt n$, where $C$ is the constant from Lemma~\ref{ldroste}.  
  
  Let $D'' \coloneqq \max\{2k+1, C_{D''}\}$, where $C_{D''}$ is a sufficiently large constant (that depends only on the constant $C$ from Lemma~\ref{ldroste}). Let $T''$ be the first time that $D_t < D''$ or that there is a frequency $f_{it}$ that is less than $\frac 13$. By Lemma~\ref{lmiddle}, we have $E[T'' - T'] = O(\mu \log n)$. Hence a simple Markov bound gives $T'' \le T' + \mu n^{0.4} \ln n$ with probability $1 - O(n^{-0.4})$. 
   
  Finally, let $c_T \coloneqq \frac{10(2 + \sqrt 2)}{C} +1$ and assume that $c_\mu \ge 144 c_T$. Using our assumption that $k \ge 17$, we first invoke Lemma~\ref{lfreqlb} with $T = c_T \mu \sqrt n$ and obtain that up to time $T''' = \min\{T'',T\}$, all frequencies are at least $\frac 13$ with probability $1 - 2n\exp(-\frac{\mu^2}{72T}) \ge 1 - 2n\exp(-\frac{\mu}{72 c_T \sqrt n}) \ge 1 - 2n\exp(-\frac{c_{\mu}}{72 c_T} \ln n) = 1 - O(\frac 1n)$ by choice of $c_\mu$.  
  
  Putting these three arguments together, we see that with probability $1 - O(\frac 1n) - O(n^{-0.4}) - O(\frac 1n) = 1 - O(n^{-0.4})$, there is a time $t = O(\mu \sqrt n)$ such that $D_t \le D'' \le \frac 1 {10} \ln(n)$ and $f_t \in [\frac 13,1]^n$. 
  By Lemma~\ref{lfinish1}, we now find the optimum in $O(\frac{\mu}{\log n})$ iterations with probability $1 - n^{-\omega(1)}$. This shows that the total runtime is $O(\mu \sqrt n)$ with probability $1 - O(n^{-0.4}) - n^{-\omega(1)} = 1 - O(n^{-0.4})$.
\end{proof}

Let us remark that we did not try to optimize the implicit constants, nor did we try to find the largest constant $C_k$ such that the $O(n \log n)$ runtime guarantee holds for all $k \le C_k \ln(n) - 1$. We further note that all  but one argument in the above proof, by choosing the constants right, would give a success probability of $1 - n^{-c}$, where $c$ can be any constant. This is not true for the Markov bound argument in the analysis of the time to reach a $D_t$ value of at most $D''$. Without further details, we note that also for this phase an arbitrary inverse-polynomial failure probability could be obtained with stronger methods. %note: probability amplification des markov-arguments. Mit Wkeit 1/2 schaffen wir's fix. Wenn nicht, dann schnell zurueck zum Start mit den Argumenten der ersten Phasen und nochmal probieren.

Finally, we note that by taking $k=1$, our result also applies to the \onemax function. 

\subsection{General Insights From This Proof}

Our result that the cGA can cross small fitness valleys at no extra cost, whereas many EAs pay an $\Omega(n^k)$ price for this, raises the question why these algorithms differ that significantly. From our proof, we obtain the following insight. 

To ease the presentation, we take as point of comparison the simple \oea, but as discussed earlier, similar behaviors are observed for many other mutation-based EAs. Again, when talking about the cGA, we measure the progress via the frequency distance $D_t = n - \|f_t\|$, which is the expected fitness distance of a sample. For the \oea, naturally, we regard the Hamming distance $d(x_t) = n - \onemax(x_t)$ of the current solution $x_t$ from the optimum. 

We observe that both algorithms easily reach a distance of $O(k)$. For the cGA this is ``only'' $O(k)$ and for the \oea this is exactly $k$, but this difference is not important. The important difference is that from such a state, the \oea samples the optimum only with probability $O(n^{-k})$, whereas the cGA does so with probability $\exp(-\Omega(k))$, at least when $f_t \in [\frac 13,1]^n$.

A first observation is that the cGA samples solutions with higher variance. This is easiest visible from Lemma~\ref{ldroste}, which implies that with constant probability the distance $d(y)$ of a sample $y$ is $\Omega(\sqrt{D_t})$ away from the expected distance $E[d(y)] = D_t$.

For the \oea, the sampling variance is much smaller. Since the number of bits that are flipped in a mution follows a binomial distribution with parameters $n$ and $\frac 1n$, which is asymptotically a Poisson distribution with parameter $\lambda = 1$, we see that larger fitness changes can only occur with relatively small probability (e.g., a super-constant fitness change happens only with probability $o(1)$, a fitness change of $\delta$ happens with probability at most $\delta^{-\Omega(\delta)}$). 

The reason for this low sampling variance of the \oea, obviously, is the small mutation rate of $\frac 1n$ usually employed. However, raising the mutation rate does not solve the problem and, in fact, creates new problems. When using a larger mutation rate, then the expected \onemax fitness of the offspring gets worse. If $x$ is a search point with distance $d(x) = k = O(\log n)$ and $y$ is obtained from $x$ via standard-bit mutation with mutation rate $p$, then the expected distance of $y$ from the optimum is $E[d(y)] = d(x) + pn(1 - 2d(x)/n)$. 

Clearly, worsening the expected quality of the offspring can only make sense if there is a clear gain from this. Unfortunately, there is no such gain. Indeed, when using a larger mutation rate $p$, then the expected distance $d(y)$ has a larger variance. However, this variance mostly works into the wrong direction. When not only looking at the first or second moment, but at the precise distribution, then we see that the distance gain or loss is distributed as $d(x) - d(y) \sim -X_{n-k,p} + X_{k,p}$, where $X_{n-k,p}$ and $X_{k,p}$ are independent random variables following binomial laws with parameters $(n-k,p)$ and $(k,p)$, respectively. Consequently, a positive gain can only stem from the $X_{k,p}$ part, which (unless $p$ is ridiculously large) again has a small variance since $k$ is small. 

In summary, we see that regardless of how we set the mutation rate, the \oea only with relatively small probability reduces the distance by a larger amount. This is caused by a generally small sampling variance when $p$ is small, say $p = \frac 1n$, or by the fact that the distribution of the distance change is highly asymmetric in the way that true distance reductions are unlikely (when $p$ is larger). 

For the cGA, things are different. Assuming for simplicity a frequency vector $f_t = (1 - \frac{2k}{n}) \textbf{1}_n$, then the fitness gain of a sample $y$ over the expectation is distributed like $D_t - d(y) \sim X_{n,\frac{2k}{n}}$, where again $X_{n,\frac{2k}{n}}$ denotes a random variable following a binomial law with parameters $n$ and $\frac{2k}{n}$. While this distribution is not perfectly symmetric, it is not too strongly concentrated in both directions and thus allows larger improvements with reasonable probability, in particular, sampling the optimum with probability $\exp(-O(k))$. This substantially different way how solutions are sampled seems to be the key to the significantly better performance of the cGA on jump functions.

\section{An Exponential Lower Bound}\label{sec:expLB}

We now prove that the cGA, regardless of the value of the parameter~$\mu$, optimizes jump functions in a time that is at least exponential in the jump size~$k$. 

As for our upper bound result, also this lower bound is valid for a broader class of functions. We say that a function $\calF : \{0,1\}^n \to \R$ is a \emph{superjump function} with jump size $k$ if it has a unique global maximum $x^*$ and for all $r \in [1..k-1]$ and $x, y \in \{0,1\}^n$ with $H(x,x^*) = r$ and $H(y,x^*) = r+1$ we have $\calF(x) < \calF(y)$; here we recall the definition 
\[H(x,y) \coloneqq \{i \in [1..n] \mid x_i \neq y_i\}\]
of the Hamming distance between the bit strings $x$ and $y$ of length $n$. In other words, $\calF$ has a unique global maximum and is fully deceptive in a ball of radius $k$ around this optimum: search points closer to the optimum have a lower fitness. Clearly, all jump functions with jump size $k$ or larger are superjump functions with jump size~$k$. Also, by arbitrarily modifying a jump function outside the gap region and in a way that the global optimum remains the unique global optimum, we obtain superjump functions.

We now show the following result.

\begin{theorem}\label{thm:main}
  There are constants $\alpha_1, \alpha_2 > 0$ such that for any $n$ sufficiently large and any $k \in [1..n]$, regardless of the hypothetical population size $\mu$, the runtime of the cGA on any superjump function with jump size $k$ with probability $1 - \exp(-\alpha_1 k)$ is at least $\exp(\alpha_2 k)$. In particular, the expected runtime is exponential in $k$.
\end{theorem}

We note that we intentionally prove a runtime bound that holds with high probability. The reason is that, as discussed in Section~\ref{sec:whp}, EDAs may with small probability reach states from which they find it very hard to reach the optimum. Such a situation could lead to a very high expected runtime even when the EDA with high probability is very efficient. For that reason, lower bounds that hold with high probability are particularly desirable for EDAs. Needless to say, a lower bound that holds with high probability immediately implies an asymptotically identical lower bound on the expected runtime.

We also note that the cGA is treating all bit-positions and the two bit values zero and one in a symmetric fashion (this property was called \emph{unbiased} in~\cite{LehreW12}). Consequently, in any runtime analysis of the cGA on a pseudo-Boolean function we can assume that $(1, \ldots, 1)$ is an optimum. Since further the actions of the cGA do not depend on absolute fitness values, but only on relative ones (this property was called \emph{ranking-based} in~\cite{DoerrW14ranking}), its performance is invariant under monotonic rescalings of the fitness function. For this reason, it suffices to regard superjump functions that agree with the jump function of jump size $k$ on all search points $x$ with $\|x\|_1 \ge n-k$ (and thus also have $(1, \dots, 1)$ as unique global optimum). To ease the presentation, we shall take this assumption in the remainder without further notice. This also allows us to continue to use the definition
\[G_{nk} \coloneqq \{x \in \{0,1\}^n \mid n-k < \|x\|_1 < n\}\] 
of the gap. 

Before stating the formal proof, we brief{}ly describe the main proof arguments on a more intuitive level. As in the previous section, we will regard the stochastic process $D_t \coloneqq n - \|f_t\|_1$, that is, the distance between the sum of the frequencies and its ideal value $n$. Our general argument is that this process with probability $1 - \exp(-\Omega(k))$ stays above $\frac 14 k$ for $\exp(\Omega(k))$ iterations. In each iteration with $D_t \ge \frac 14 k$, the probability that the optimum is sampled is only $\exp(-\Omega(k))$, see Lemma~\ref{loptUB}. Hence there is a $T = \exp(\Omega(k))$ such that with probability $1-\exp(-\Omega(k))$, the optimum is not sampled in the first $T$ iterations.

The heart of the proof is an analysis of the process $(D_t)$. It is intuitively clear that once the process is below $k$, then often the two search points sampled in one iteration both lie in the gap region, which gives $D_t$ a positive drift (that is, a decrease of the average frequency). To turn this drift away from the target (a small $D_t$ value) into an exponential lower bound on the runtime, we consider the process 
\[Y_t = \exp(c \min\{\tfrac 12 k - D_t, \tfrac 14 k\}),\] that is, an exponential rescaling of $D_t$. Such a rescaling has recently also been used in~\cite{AntipovDY19}. We note that the usual way to prove exponential lower bounds is the negative drift theorem of Oliveto and Witt~\cite{OlivetoW12}. We did not immediately see how to use it for our purposes, though, since in our process we do not have very strong bounds on the one-step differences. E.g., when $D_t = \frac 12 k$, then the underlying frequency vector may be such that $D_{t+1} \ge D_t + \sqrt k$ happens with constant probability. We also note that after the submission of this work, a negative multiplicative drift theorem was proposed~\cite{Doerr20ppsnLB}, which would be applicable to our setting as well. It would, however, not greatly simplify the proof as the main work, estimating the drift of the process $(Y_t)$, would still be needed. 

We shall show that the process $Y_t$ has at most a constant point-wise drift, more precisely, that 
\begin{equation}
  E[Y_{t+1} - Y_t \mid Y_t = y] \le 2 \label{eq:driftzwei}
\end{equation} 
holds for all $y < Y_{\max} \coloneqq \exp(\frac c4 k)$. From this statement, the lower bound version of the additive drift theorem (Theorem~\ref{tdrift}\ref{it:driftlb}) would immediately show that the expected time to reach a $D_t$ value of $\frac k4$ or less is at least exponential in $k$. However, since we aim at a runtime bound that holds with high probability, we take a different (and, in fact, more elementary) route. We regard the process $(\tilde Y_t)$ which is identical to $Y$ until $Y$ first reaches $Y_{\max}$ and then stays constant at $Y_{\max}$. This process satisfies $E[\tilde Y_{t+1} - \tilde Y_t] \le 2$ for all times $t$. From this and $\tilde Y_0 = Y_0 < 1$ we obtain $E[\tilde Y_t] \le 1+2t$. Hence for $T = \exp(\Omega(k))$ sufficiently small, we have 
\[\frac{E[\tilde Y_T]}{Y_{\max}} = \exp(-\Omega(k))\] 
and a simple Markov bound argument is enough to show that ${\Pr[\tilde Y_T = Y_{\max}]} = \exp(-\Omega(k))$. Note that $\tilde Y_T < Y_{\max}$ is equivalent to $Y_t < Y_{\max}$ for all $t \in [0..T]$.

  The main work in the following proof is showing~\eqref{eq:driftzwei}. The difficulty here is hidden in a small detail. When $D_t \in [\frac 14 k, \frac 34 k]$, and this is the most interesting case (case~2 in the formal proof), then we have $\|f'_{t+1}\|_1 \le \|f_t\|$ whenever the two search points sampled lie in the gap region, and hence with probability $1 - \exp(-\Omega(k))$; from Lemma~\ref{ldiff} we obtain, in addition, a true decrease, that is, $\|f'_{t+1}\|_1 \le \|f_t\| - \frac 1 \mu$, with constant probability. This progress of $f'_{t+1}$ over $f_t$ would be perfectly fine for our purposes. Hence the true difficulty arises from the capping of the frequencies into the interval $[\frac 1n,1-\frac 1n]$, that is, from the fact that the new frequency vector is $f_{t+1} \coloneqq \minmax(\frac 1n \textbf{1}_n, f'_{t+1}, (1- \frac 1n) \textbf{1}_n)$. This appears to be a minor problem, among others, because only a capping at the lower bound $\frac 1n$ can have an adverse effect on our process, and there are at most $O(k)$ frequencies sufficiently close to the lower boundary. Things become difficult due to the exponential scaling, which can let rare events still have a significant influence on the expected change of the process.   

We now make these arguments precise and prove Theorem~\ref{thm:main}. We recall that, while the theorem refers to arbitrary superjump functions with jump size $k$, we can always assume that we regard a fitness function that agrees with the classic jump function with parameter $k$ on all search points $x$ with $\|x\|_1 \ge n-k$, including the unique global optimum $(1, \dots, 1)$.

\begin{proof}
  Since we are aiming at an asymptotic statement, we can assume in the following that $n$ is sufficiently large.
  
  To ease the presentation of the main part of the proof, let us first give a basic argument for the case of small $k$ and then assume that $k \ge w(n)$ for some function $w : \N \to \N$ with $\lim_{n \to \infty} w(n) = \infty$. 
  
  We first note that with probability $f_{it}^2 + (1 - f_{it})^2 \ge \frac 12$, the two search points $x^1$ and $x^2$ generated in the $t$-th iteration agree in the $i$-th bit, which in particular implies that $f_{i,t+1} = f_{it}$. Hence with probability at least $2^{-T}$, this happens for the first $T$ iterations, and thus $f_{it} = \frac 12$ for all $t \in [0..T]$. Let us call such a bit position $i$ \emph{idle}. 
  
  Note that the events of being idle are independent for all $i \in [1..n]$. Hence, taking $T = \lfloor \frac 12 \log_2 n \rfloor$, we see that the number $X$ of idle positions has an expectation of $E[X] \ge n 2^{-T} \ge \sqrt n$, and by a simple Chernoff bound (Theorem~\ref{tchernoff}), we have $\Pr[X \ge \frac 12 \sqrt n] \ge 1 - \exp(-\Omega(\sqrt n))$. 
  
  Conditional on having at least $\tfrac 12 \sqrt n$ idle bit positions, the probability that a particular search point sampled in the first $T$ iterations is the optimum is at most $2^{-\frac 12 \sqrt n}$. By a simple union bound argument, the probability that at least one of the search points generated in the first $T$ iterations is the optimum is at most $2 T 2^{-\frac 12 \sqrt n} = \exp(-\Omega(\sqrt n))$. In summary, we have that with probability at least $1 - \exp(-\Omega(\sqrt n))$, the runtime of the cGA on any function with unique optimum (and in particular any superjump function) is greater than $T = \frac 12 \log_2 n$. This implies the claim of this theorem for any $k \le C \log \log n$, where $C$ is a sufficiently small constant, and, as discussed above, $n$ is sufficiently large.
  
  With this, we can now safely assume that $k = \omega(1)$. Since $k \le n$ and our result is invariant under constant-factor changes of $k$, we can assume that $k \le \frac{n}{320}$.
%  For the case that $k \ge \frac n {320}$, we will need slightly modified calculations. To keep this proof readable, we hide (but treat nevertheless) this case as follows. We consider a run of the cGA on a jump function $\jump_{nk'}$ with $k' \in [1..n] \cap \omega(1)$ arbitrary and we let $k \coloneqq \min\{k', \frac n {320}\}$. %As for all other constants in this proof, we did not try to optimize them as they only have an influence on the constant hidden in the asymptotic notation of our result. 
  
  Let $D_t \coloneqq n - \|f_t\|_1 = n - \sum_{i=1}^n f_{it}$ be the distance of the sum of the frequencies from the ideal value $n$. 
  
  Our intuition (which will be made precise) is that the process $(D_t)$ finds it hard to go significantly below $k$ because there we will typically sample individuals in the gap, which leads to a decrease of the sum of frequencies (when the two individuals have different distances from the optimum). To obtain an exponential lower bound on the runtime, we suitably rescale the process by defining, for a sufficiently small constant $c$, 
  \[Y_t = \min\{\exp(c(\tfrac 12 k - D_t)),\exp(\tfrac 14 c k)\} = \exp(c \min\{\tfrac 12 k - D_t, \tfrac 14 k\}).\] 
  Observe that $Y_t$ attains its maximal value $\Ymax = \exp(\frac 14 c k)$ precisely when $D_t \le \frac 14k$. Also, $Y_t \le 1$ for $D_t \ge \frac 12 k$.
  
  To argue that we have $D_t > \tfrac 14 k$ for a long time, we now show that for all $y < \Ymax$ the drift $E[Y_{t+1} - Y_t \mid Y_t = y]$ is at most constant. To this aim, we condition on a fixed value of $f_{t}$, which also determines~$D_t$. We treat separately the two cases that $D_t \ge \frac 34 k$ and that $\frac 34 k > D_t > \frac 14 k$.
   
  \textbf{Case 1:} Assume first that $D_t \ge \frac 34 k$. By Lemma~\ref{lsample}, with probability $1 - \exp(-\Omega(D_t)) \ge 1 - \exp(-\Omega(k))$, the two search points $x^1,x^2$ sampled in iteration $t+1$ both satisfy 
  \begin{equation}
    \big|\|x^i\|_1 - \|f_t\|_1\big| = |d(x^i) - D_t| < \tfrac 1 {18} D_t \le \tfrac 16 (D_t - \tfrac 12 k). \label{eq:A}
  \end{equation}
  Here and in the following, when writing $\Omega(k)$ we mean that there is a positive constant $C$, independent of $n$, $k$, and $c$, such that the expression is at least~$Ck$. Let us call $A$ the event described in~\eqref{eq:A}. In this case, we argue as follows. We recall the notation $\sum[v] \coloneqq \sum_{i=1}^n v_i$ to denote the sum of the elements of an $n$-dimensional vector~$v$ and we recall further that, with a slight abuse of notation, we defined $\|f'\|_1 \coloneqq \sum[f']$ for intermediate frequency vectors $f'$. Let $\{y^1,y^2\} = \{x^1,x^2\}$ such that $\calF(y^1) \ge \calF(y^2)$. Then 
  \begin{align*}
  \|f'_{t+1}\|_1 &= \sum\Big[f_t + \tfrac 1\mu (y^1-y^2)\Big] \\
  &= \sum\Big[f_t + \tfrac 1\mu (y^1-f_t) - \tfrac 1\mu(y^2-f_t)\Big] \\
  &\le \sum\Big[f_t\Big] + \tfrac 1\mu \left|\sum\Big[y^1-f_t\Big]\right| + \tfrac 1\mu \left|\sum\Big[y^2-f_t\Big]\right| \\
  &= \|f_t\|_1 + \tfrac 1\mu \big|\|x^1\|_1 - \|f_t\|_1\big| + \tfrac 1\mu \big|\|x^2\| - \|f_t\|_1\big| \\ 
  &\le n-D_t + 2 \tfrac 1\mu \tfrac 16 (D_t - \tfrac 12 k) \\
  &\le n-D_t + 2 \tfrac 16 (D_t - \tfrac 12 k)  \\
  &= n - \tfrac 23 D_t - \tfrac 16k \le n - \tfrac 23 \cdot \tfrac 34 k - \tfrac 16k \le n - \tfrac 23 k.
  \end{align*} 
  
  We still need to consider the possibility that $f_{i,t+1} > f'_{i,t+1}$ for some $i \in [1..n]$. By Lemma~\ref{lboundary}, not conditioning on $A$, we have that $\|f_{t+1}\|_1 - \|f'_{t+1}\|_1 \preceq \frac 1\mu \Bin(\ell,P) \preceq \Bin(\ell,P)$ for some $\ell \in [1..n]$ and $P = 2 \frac 1n (1-\frac 1n)$.
  
   Let us call $B$ the event that $\|f_{t+1}\|_1 - \|f'_{t+1}\|_1 < \frac 16 k$. Note that $A \cap B$ implies $\|f_{t+1}\|_1 < n- \frac 12 k$ and thus $Y_{t+1} \le 1$. By Lemma~\ref{lprobbino} and the estimate $\binom{a}{b} \le (\frac{ea}{b})^b$, we have 
   \[\Pr[\neg B] \le \binom{\ell}{\frac 16 k} P^{k/6} \le \left(\frac{12e\ell}{kn}\right)^{k/6} \le k^{-\Omega(k)}.\] 
   
  We conclude that the event $A \cap B$ holds with probability $1 - \exp(-\Omega(k))$; in this case $Y_t \le 1$ and $Y_{t+1} \le 1$. In all other cases,  we bluntly estimate $Y_{t+1} - Y_t \le \Ymax$. This gives 
  \[E[Y_{t+1} - Y_t] \le (1 - \exp(-\Omega(k))) \cdot 1 + \exp(-\Omega(k)) \Ymax.\] 
  By choosing the constant $c$ in the definition of $(Y_t)$ sufficiently small and taking $n$ sufficiently large, we have $E[Y_{t+1} - Y_t] \le 2$.

  \textbf{Case 2:} Assume now that $\frac 34 k > D_t > \frac 14k$. Let $x^1,x^2$ be the two search points sampled in iteration $t+1$ and let $y^1, y^2$ be such that $\{y^1,y^2\} = \{x^1,x^2\}$ and $\calF(y^1) \ge \calF(y^2)$. By Lemma~\ref{lsample} again, we have ${k > n - \|x^i\|_1 > 0}$ with probability $1 - \exp(-\Omega(k))$ for both $i \in \{1,2\}$. Let us call this event $A$. Note that if $A$ holds, then both offspring lie in the gap region. Consequently, $\|y^1\|_1 \le \|y^2\|_1$ and thus $\|f'_{t+1}\|_1 \le \|f_t\|_1$.
    
    Let $L = \{i \in [1..n] \mid f_{it} = \frac 1n\}$, $\ell = |L|$, and $M = \{i \in L \mid x^1_i \neq x^2_i\}$ as in Lemma~\ref{lboundary}. Note that by definition, $D_t \ge (1-\frac 1n) \ell$, hence from $D_t < \frac 34 k$ and $n \ge 4$ we obtain $\ell < k$. 
   
   Let $B_0$ be the event that $|M|=0$, that is, $x^1_{|L} =  x^2_{|L}$. Note that in this case, we have $f_{t+1} \le f'_{t+1}$ (component-wise) and thus 
   \[\|f_{t+1}\|_1 \le \|f'_{t+1}\|_1 = \|f_t + \tfrac 1 \mu (y^1 - y^2)\|_1 = \|f_t\|_1 + \tfrac 1 \mu (\|y^1\|_1 - \|y^2\|_1).\] By Lemma~\ref{lboundary}, Bernoulli's inequality, and $\ell \le k$, we have 
   \[\Pr[B_0] = (1 - 2\tfrac 1n (1-\tfrac 1n))^{\ell} \ge 1 - \tfrac{2\ell}{n} \ge 1 - \tfrac{2k}n.\] 
   %technique: hier geht $k < n/2$ ein
   Since $\ell < k \le \frac{n}{320} < \frac n2$, by Lemma~\ref{ldiff}, we have $\|x^1_{|[n]\setminus L}\|_1 \neq \|x^2_{|[n]\setminus L}\|_1$ with probability at least $\frac{1}{16}$. This event, called $C$ in the following, is independent of $B_0$. We have 
   \[\Pr[A \cap B_0 \cap C] \ge \Pr[B_0 \cap C] - \Pr[\overline A] \ge (1-\tfrac{2k}{n})\tfrac{1}{16} - \exp(-\Omega(k)).\] 
   If $A \cap B_0 \cap C$ holds, then $\|f_{t+1}\|_1 \le \|f'_{t+1}\|_1 \le \|f_t\|_1 - \frac 1 \mu$. If $A \cap B_0 \cap \overline C$ holds, then we still have $\|f_{t+1}\|_1 \le \|f'_{t+1}\|_1 \le \|f_t\|_1$. 
   
  Let us now, for $j \in [1..\ell]$, denote by $B_j$ the event that $|M|=j$, that is, that $x^1_{|L}$ and $x^2_{|L}$ differ in exactly $j$ bits. By Lemma~\ref{lboundary} again, we have $\Pr[B_j] = \Pr[\Bin(\ell,P) = j]$.
  
  The event $A \cap B_j$ implies $\|f_{t+1}\|_1 \le \|f'_{t+1}\|_1 + \frac j \mu  \le \|f_t\|_1 + \frac j \mu$ and occurs with probability $\Pr[A \cap B_j] \le \Pr[B_j] = \Pr[\Bin(\ell,P) = j]$.
  
  Taking these observations together, we compute
  \begin{align}
  E[Y_{t+1}] 
  &= \Pr[\overline A] \, E[Y_{t+1} \mid \overline A] \nonumber\\
  &\quad + \sum_{j=1}^\ell \Pr[A \cap B_j] \, E[Y_{t+1} \mid A \cap B_j] \nonumber\\
  &\quad + \Pr[A \cap B_0 \cap \overline C] \, E[Y_{t+1} \mid A \cap B_0 \cap \overline C] \nonumber\\
  &\quad + \Pr[A \cap B_0 \cap C] \, E[Y_{t+1} \mid A \cap B_0 \cap C] \nonumber\\
  &\le \exp(-\Omega(k)) \Ymax \label{eq:bilanz}\\
  &\quad + \sum_{j=1}^\ell \Pr[\Bin(\ell,P) = j] \, Y_t \exp(\tfrac{cj}{\mu}) \nonumber\\
  &\quad + \Pr[\Bin(\ell,P) = 0] \, Y_t\nonumber\\
  &\quad - (\tfrac 1 {16} (1 - \tfrac{2k}{n}) - \exp(-\Omega(k))) Y_t (1 - \exp(-\tfrac{c}{\mu})). \nonumber
%  	\\
%  &\le 1 + O\left(\frac{l}{n}\right) Y_t + Y_t - \Omega(1) Y_t 
  \end{align} 
We note that the second and third term amount to $Y_t E[\exp(\frac{cZ}{\mu})]$, where $Z \sim \Bin(\ell,P)$. Writing $Z = \sum_{i=1}^\ell Z_i$ as a sum of $\ell$ independent binary random variables with $\Pr[Z_i = 1] = P$, we obtain 
\[E[\exp(\tfrac{cZ}{\mu})] = \prod_{i=1}^\ell E[\exp(\tfrac{cZ_i}{\mu})] = (1 - P + P\exp(\tfrac{c}{\mu}))^\ell.\]
By assuming $c \le 1$ and using the elementary estimate $e^x \le 1 + 2x$ valid for $x \in [0,1]$, see, e.g., Lemma~1.4.2(b) in~\cite{Doerr20bookchapter}, we have 
\[1 - P + P\exp(\tfrac{c}{\mu}) \le 1 + 2P(\tfrac{c}{\mu}).\] 
Hence with $P \le \frac 2n$, $\mu \ge 1$, and $\ell \le \frac n{320}$, we obtain 
\[E[\exp(\tfrac{cZ}{\mu})] \le (1 + 2P(\tfrac{c}{\mu}))^\ell \le \exp(2P(\tfrac{c}{\mu})\ell) \le \exp(\tfrac{4c}{320\mu}) \le 1 + \tfrac{c}{40\mu},\]
again by using $e^x \le 1 + 2x$. The second and third term of~\eqref{eq:bilanz} thus add up to at most $(1 + \frac{c}{40}) Y_t$.

In the first term of~\eqref{eq:bilanz}, we again assume that $c$ is sufficiently small to ensure that $\exp(-\Omega(k)) \Ymax = \exp(-\Omega(k)) \exp(\frac 14 ck) \le 1$. Recalling that $k \le \frac{n}{320}$ and assuming $k$ sufficiently large (since $k = \omega(1)$ and $n$ is large), we finally estimate in the last term $\tfrac 1 {16} (1 - \tfrac{2k}{n}) - \exp(-\Omega(k)) \ge \tfrac 1{20}$ and, more interestingly, $1 - \exp(-\tfrac{c}{\mu}) \ge \tfrac{c}{\mu}(1-\frac 1e)$ using the estimate $e^{-x} \le 1-x(1-\frac 1e)$ valid for all $x \in [0,1]$, which stems simply from the convexity of the exponential function.

With these estimates we obtain 
\[E[Y_{t+1}] \le 1 + (1 + \tfrac{c}{40\mu}) Y_t - \tfrac 1 {20} (1-\tfrac 1e) \tfrac{c}{\mu} Y_t \le 1 + Y_t\] 
and thus $E[Y_{t+1} - Y_t] \le 1$.

In summary, we have now shown that for all $y < \Ymax$ and at all times $t$ the process $(Y_t)$ satisfies $E[Y_{t+1} - Y_t \mid Y_t = y] \le 2$. We note that $Y_0 \le 1$ with probability one. For the sake of the argument, let us artificially modify the process from the point on when it has reached a state of at least $Y_{\max}$. So we define $(\tilde Y_t)$ by setting $\tilde Y_t = Y_t$, if $Y_t < Y_{\max}$ or if $Y_t \ge Y_{\max}$ and $Y_{t-1} < Y_{\max}$, and $\tilde Y_t = \tilde Y_{t-1}$ otherwise. In other words, $(\tilde Y_t)$ is a copy of $(Y_t)$ until it reaches a state of at least $Y_{\max}$ and then does not move anymore. With this trick, we have $E[\tilde Y_{t+1} - \tilde Y_t] \le 2$ for all $t$. 

A simple induction and the initial condition $\tilde Y_0 \le 1$ shows that ${E[\tilde Y_t] \le 2t+1}$ for all $t$. In particular, for $T = \frac 12 \exp(\frac 18 c k) - 1$, we have $E[Y_T] \le \exp(\frac 18 c k)$ and, by Markov's inequality, 
\[\Pr[\tilde Y_T \ge Y_{\max}] \le \frac{\exp(\frac 18 c k)}{Y_{\max}} = \exp(-\tfrac 18 c k).\]
  
Hence with probability $1 - \exp(-\tfrac 18 c k)$, we have $\tilde Y_T < Y_{\max}$. We now condition on this event. By construction of $(\tilde Y_t)$, we have $Y_t < Y_{\max}$, equivalently $D_t > \frac 14 k$, for all $t \in [0..T]$. If $D_t > \frac 14 k$, then by Lemma~\ref{loptUB} the probability that a sample generated in this iteration is the optimum, is at most $\exp(-\tfrac 14 k)$. Assuming $c \le 1$ again, we see that the probability that the optimum is generated in one of the first $T$ iterations, is at most $2 T \exp(-\frac 14 k) \le \exp(\frac 18 c k) \exp(-\frac 14 k) = \exp(-\frac 18 k)$. This shows the claim. 
\end{proof}

\section{An $\Omega(n \log n)$ Lower Bound}\label{sec:nlogn}

With the exponential lower bound proven in the previous section, the runtime of the cGA on jump functions is well understood, except that the innocent looking lower bound $\Omega(n \log n)$, matching the corresponding upper bound for $k \le \frac 1 {20} \ln n -1$ and optimal choice of $\mu$, is still missing. Since Sudholt and Witt~\cite{SudholtW19} have proven an $\Omega(n \log n)$ lower bound for the simple unimodal function $\onemax$, which for many EAs is known to be one of the easiest functions with unique global optimum~\cite{DoerrJW12algo,Sudholt13,Witt13,Doerr19tcs}, it would be very surprising if this lower bound would not hold for jump functions as well. 

In this section, we first argue why, unlike for many other algorithms, it is hard to show that a lower bound on the runtime of the cGA on \onemax extends to a lower bound for any other function with unique optimum. We then analyze in detail the proof of the $\Omega(n \log n)$ lower bound for \onemax~\cite{SudholtW19} and argue that the same arguments can be applied in the case of jump functions (but not superjump functions).

\subsection{Domination Arguments Fail}

The true reason why \onemax is the easiest optimization problem for many evolutionary algorithms $\calA$, implicit in all such proofs and explicit in~\cite{Doerr19tcs}, is that when comparing a run of $\calA$ on \onemax and on some other function $\calF$ with unique global optimum, then at all times the Hamming distance between the current-best solution and the optimum in the \onemax process is stochastically dominated by the same quantity in the other process. This follows by induction and a coupling argument from the following key insight (here formulated for the \oea only).

\begin{lemma}\label{lem:dom}
  Let $\calF : \{0,1\}^n \to \R$ be some function with unique global optimum $x^*$ and let \onemax be the $n$-dimensional \onemax function with unique global optimum $y^* = (1, \dots, 1)$. Let $x,y \in \{0,1\}^n$ such that $H(x,x^*) \ge H(y,y^*)$, where $H(\cdot, \cdot)$ denotes the Hamming distance. Consider one iteration of the \oea optimizing $\calF$, started with $x$ as parent individual, and denote by $x'$ the parent in the next iteration. Define $y'$ analogously for \onemax and $y$. Then $H(x',x^*) \succeq H(y',y^*)$.
\end{lemma}

As a side remark, note that the lemma applied in the special case $\calF = \onemax$ shows that the intuitive rule ``the closer a search point is to the optimum, the shorter is the optimization time when starting from this search point'' holds for optimizing \onemax via the \oea.

We now show that a statement like Lemma~\ref{lem:dom} is not true for the cGA. Since the states of a run of the cGA are the frequency vectors $f$, the natural extension of the Hamming distance quality measure above is the $\ell_1$-distance $d(f,x^*) = \|f-x^*\|_1 = \sum_{i=1}^n |f_i - x^*_i|$. Note that for $x^* = (1, \ldots, 1)$, we have $d(f,x^*) = n - \|f\|_1$, the distance measure regarded in many of the other proofs in this work. 

\begin{lemma}
  Let $n$ be even. We consider running the cGA with hypothetical population size $\mu = n$. Then there are a fitness function $\calF : \{0,1\}^n \to \R$ with unique global optimum $x^* = (1, \ldots, 1)$ and frequency vectors $f, g \in (F_\mu)^n$ such that the following holds. Let $\tilde f$ be the frequency vector obtained after one iteration of optimizing $\calF$ via the cGA started with frequency vector $f$. Let $\tilde g$ be the frequency vector obtained after one iteration running the cGA on \onemax (with unique global optimum $y^* \coloneqq x^*$) started with~$g$. Then $d(f,x^*) \ge d(g,y^*)$, but $d(\tilde f,x^*) \not\succeq d(\tilde g,y^*)$.  
\end{lemma}

\begin{proof}
  Let $\calF$ be any subjump function with jump size $k \le \frac n4$. Let $f = \frac 12 \textbf{1}_n$. Let $g \in [0,1]^n$ be such that half the entries of $g$ are equal to $\frac 1n + \frac 1\mu = \frac 2n$ and the other half are equal to $1 - \frac 1n - \frac 1\mu = 1 - \frac 2n$. 

We obviously have $d(f,x^*) \ge d(g,y^*)$, since both numbers are equal to~$\frac n2$. Since with probability $1 - \exp(-\Omega(n))$, both search points sampled in the jump process have between $\frac n4$ and $\frac 3{4n}$ ones, their jump fitnesses equal their \onemax fitnesses. Consequently, we may apply Lemma~5 from~\cite{Droste06} (or, with one more argument, Lemma~\ref{ldroste}) and see that $E[d(\tilde f,x^*)] \le \frac n2 - \Omega(\frac 1 \mu \sqrt n)$. For the \onemax process started in $g$, however, denoting the two search points generated in this iteration by $x^1$ and $x^2$, we have 
\[E[\|g - \tilde g\|_1] = \sum_{i=1}^n \tfrac 1 \mu \Pr[x_i^1 \neq x_i^2] = n \cdot \tfrac 1 \mu \cdot 2 \cdot \tfrac 2n (1 - \tfrac 2n) \le \tfrac 4n.\]
From this and $d(\tilde g,y^*) = \|\tilde g-y^*\|_1 = \|\tilde g - g + g - y^*\|_1 \ge \|g - y^*\|_1 - \|\tilde g - g\|_1$, we obtain 
 \[E[d(\tilde g,y^*)] \ge d(g,y^*) - \tfrac 4 {\mu} = \tfrac n2 - O(\tfrac 1 {\mu}).\] Since thus $E[d(\tilde f,x^*)] \le E[d(\tilde g,y^*)]$, we cannot have $d(\tilde f,x^*) \succeq d(\tilde g,y^*)$.
\end{proof}

We note that a second imaginable domination result is also not true, namely that, roughly speaking, the frequency vector arising from one iteration started with a better initial frequency vector dominates the result of starting with a worse initial frequency vector. More precisely, we have the following. 

\begin{lemma}
  Let $\mu$ be an arbitrary hypothetical population size for all cGAs considered here. There are frequency vectors $f, g \in (F_\mu)^n$ with $f \le g$ (componentwise) such that the following holds. Let $\calF$ be any subjump function with jump size at most $\frac n2$ (including the \onemax function). Let $\tilde f$ be the frequency vector resulting from optimizing $\calF$ for one iteration with the cGA started with frequency vector $f$. Let $\tilde g$ be the frequency vector resulting from optimizing $\onemax$ for one iteration with the cGA started with frequency vector $g$. Then we do not have $\tilde f_i \preceq \tilde g_i$ for all $i \in [1..n]$.
\end{lemma}

\begin{proof}  
  Let  $f = (\frac 12, \frac 1n, \dots, \frac 1n)$ and $g = \frac 12 \textbf{1}_n$. Clearly, $f \le g$. 
  
  When performing one iteration of the cGA on $\calF$ started with $f$, and denoting the two samples by $x^1$ and $x^2$ and their quality difference in all but the first bit by $\Delta = \|x^1_{|[2..n]}\|_1 - \|x^2_{|[2..n]}\|_1$, then the argument that with probability $1 - \exp(-\Omega(n))$ this iteration equals an iteration with \onemax as objective function shows that the resulting frequency vector $\tilde f$ satisfies 
\begin{align}
  \Pr[&\tilde f_1 = \tfrac 12 + \tfrac 1 \mu] \nonumber\\
  &\ge \Pr[x^1_1 \neq x^2_1] (\tfrac 12 \Pr[\Delta \notin \{-1,0\}] + \Pr[\Delta \in \{-1,0\}]) - \exp(-\Omega(n)) \nonumber\\ 
&= \Pr[x^1_1 \neq x^2_1] (\tfrac 12 + \tfrac 12 \Pr[\Delta \in \{-1,0\}]) - \exp(-\Omega(n)).\label{eq:counterex}
\end{align}
Since $\Pr[\Delta \in \{-1,0\}] \ge \Pr[\|x^1_{|[2..n]}\|_1 = \|x^2_{|[2..n]}\|_1 = 0] = (1 - \frac 1n)^{2(n-1)} \ge \frac 1{e^2}$, we have $\Pr[\tilde f_1 = \tfrac 12 + \tfrac 1 \mu] \ge \frac 14 + \frac 1 {4e^2} - \exp(-\Omega(n))$.

When starting the iteration with $g$, the resulting frequency vector $\tilde g$ satisfies an equation analogous to~\eqref{eq:counterex}, but now $\Delta$ is the difference of two binomial distributions with parameters $n-1$ and $\frac 12$. Hence, we have $\Pr[\Delta \in \{-1,0\}] = O(n^{-1/2})$, see, e.g., \cite[Lemma~1.4.13]{Doerr20bookchapter} for this elementary estimate, and thus $\Pr[\tilde g_1 = \tfrac 12 + \tfrac 1 \mu] = \frac 14 + o(1)$, disproving that $\tilde f_1 \preceq \tilde g_1$.
\end{proof}

In summary, the richer mechanism of building a probabilistic model of the search space in the cGA (as opposed to using a population in EAs) makes is hard to argue that \onemax is the easiest function for the cGA. This, in particular, has the consequence that lower bounds for the runtime of the cGA on \onemax cannot be easily extended to other functions with a unique global optimum. 

\subsection{Imitating the OneMax Proof}

Above, we have seen that a simple, general argument why a lower bound for the runtime of the cGA on \onemax should extend to jump functions appears hard to find. For this reason, we now analyze the proof of the lower bound given in~\cite{SudholtW19} and observe, fortunately, that its main arguments apply equally well to jump functions. Since the full proof in~\cite{SudholtW19} is relatively long, namely more than twelve pages, we apologize to the reader that we cannot give a self-contained version of the proof, but that instead we only argue why the arguments given in~\cite{SudholtW19} remain valid in our case. 

We show the following result, which is independent from the jump size~$k$. This result, in particular, shows that our upper bound of Theorem~\ref{thm:upper} is asymptotically tight. We note that this result is proven only for jump functions,  but not also for superjump functions. This is due to the fact that the lower bound in~\cite{SudholtW19} is only proven for \onemax and not for all functions with unique global optimum.

\begin{theorem}
  Let $c > 0$ be an arbitrary constant. Let $C$ be a constant that is sufficiently large compared to $c$. Let $\mu \ge C \log n$ and $\mu \le n^c$. Then with probability $1 - o(1)$, the runtime of the cGA with hypothetical population size $\mu$ on any $n$-dimensional jump function is at least $\Omega(\mu \sqrt n + n \log n)$.
\end{theorem}

\begin{proof} 
  When $k$ is $\Omega(n)$, then Theorem~\ref{thm:main} gives a lower bound of $\exp(\Omega(n))$ with high probability. For this reason, we can now conveniently assume that $k \le \kappa n$ for an arbitrarily small constant $\kappa > 0$.
  
  As announced, we argue that the main arguments of the proof of the corresponding result in~\cite{SudholtW19}, Theorem~8, remain valid. The proof of this Theorem~8 mostly consists of Lemma~10 to~15 (in~\cite{SudholtW19}). There is nothing to show for Lemma~10 as it refers only to iterations in which a fixed bit is performing a random-walk step (in which the fitness function is irrelevant). Lemma~11 is a statement on sums of independent random variables and does not refer to the cGA at all. In Lemma~12, a lower bound on the probability of a non-random-walk step is given. Informally speaking, a non-random-walk step for a particular bit means that in this iteration, the particular bit has an influence on how the two offspring are sorted before the frequency update. Since two search points have the same $\onemax$ value if and only if they have the same objective value w.r.t.\ some jump function, this probability for a non-random-walk step is the same for \onemax and the jump function. Lemma~13, while formulated in the language of the cGA, is a statement on independent parallel unbiased random walks. The basic argument in the proof of Lemma~14 is that when $n^\eps$ frequencies have reached the lower boundary, then with high probability at least one of them will not move for $\Omega(n \log n)$ iterations, simply because the two offspring generated in each iteration always agree in this bit. The claim of Lemma~15 includes that $\Omega(n)$ frequencies stay in the interval $[\frac 16, \frac 56]$ for a given time frame $T$. To sample a search point in the gap, since $k$ is sufficiently small, at least a constant fraction of these bits have to be sampled as one. By a simple Chernoff bound (Theorem~\ref{tchernoff}), this happens only with probability $\exp(-\Omega(n))$ in one iteration. Since Lemma~15 gives a statement with probability $1 - \poly(n) 2^{-\Omega(\min\{\mu,n\})}$ only, the probabilities of sampling a search point in the gap do not affect the failure probability of $\poly(n) 2^{-\Omega(\min\{\mu,n\})}$. The main proof of Theorem~8 consists mostly of applications of these intermediate results. Only the last two paragraphs discuss what happens after the time frame~$T$, which was analyzed in Lemma~15. These two paragraphs, however, again only use general properties of the cGA that are independent of the particular fitness function.\footnote{To be very precise, the argument that a frequency at the lower boundary leaves this boundary only with probability $O(n^{-3/2})$ in one iteration is not correct, but the authors of~\cite{SudholtW19} convinced us that also with the correct estimate of $O(\frac 1n)$ and setting the implicit constants right, at least $\sqrt n$ frequencies remain at the lower boundary at the end of the first $T$ iterations. This is enough to apply Lemma~14.} In summary, all arguments given in the proof of Theorem~8 in~\cite{SudholtW19} are equally valid for the optimization of a jump function with $k \le \kappa n$ instead of the \onemax function. This proves our claim.
\end{proof}

We note that the proof above (and thus our result) applies not only to jump functions, but to all functions where Theorem~\ref{thm:main} can be employed and, more interestingly, to all functions that agree with \onemax on all search point $x$ with $\frac n6 \le \|x\|_1 \le \frac {5n}{6}$. This restriction is necessary to use the arguments of~\cite{SudholtW19}. Overcoming this restriction is most likely non-trivial. It would most likely immediately imply a general lower bound of $\Omega(n \log n)$ for the runtime of the cGA on any function with unique global optimum, which is a major open problem in the field. 

\section{Conclusion}

This study (including the preliminary versions~\cite{Doerr19foga,Doerr19gecco}) is, to the best of our knowledge, after~\cite{HasenohrlS18} only the second mathematical analysis of an EDA on a multimodal optimization problem. Our two main results are
\begin{enumerate}
\item that the cGA can optimize jump functions with logarithmic jump sizes in asymptotically the same efficiency as the simple \onemax function; it thus does not suffer from the fitness valleys present in these objective functions;
\item an $\exp(\Omega(k))$ lower bound for the runtime of the cGA on jump functions with jump size $k$, regardless of the hypothetical population size $\mu$. This result shows, in particular, that the corresponding upper bound by Hasen\"ohrl and Sutton~\cite{HasenohrlS18} cannot be improved by running the cGA with a hypothetical population size that is sub-exponential in $k$. 
\end{enumerate}

The obvious question arising from this work is whether similar results hold for other EDAs and other optimization problems, or whether this result is a particularity of the cGA and jump functions. Natural candidates for other EDAs could be the UMDA, for which several rigorous runtime results exist, see~\cite{KrejcaW20bookchapter}, and the significance-based cGA~\cite{DoerrK20tec}, which might profit from using only the three frequencies $\frac 1n$, $\frac 12$, and $1 - \frac 1n$. Candidates for optimization problems leading to a multimodal fitness landscape include the maximum matching problem~\cite{GielW03,GielW04} or the minimum vertex cover problem~\cite{OlivetoHY09,JansenOZ13}. 

We also proved an $\Omega(n \log n)$ lower bound for jump functions in Section~\ref{sec:nlogn}, and did so by arguing that this lower bound is witnessed in the \onemax process at a time up to which the cGA most likely has not sampled a search point that lies in the gap of a jump function. For this reason, the proof of~\cite{SudholtW19} extends to jump functions as well. This argument was sufficient for our purposes, but left the real (and most likely very difficult) question untouched, namely if $\Omega(n \log n)$ is a lower bound for the cGA optimizing any function with unique global optimum. We do not dare to speculate what is the answer.

%\merk{weg/woanders (FOGA):} What is noteworthy in our proof is that it does not require a distinction between the different cases that frequencies reach boundary values or not (as in, e.g., the highly technical lower bound proof for \onemax in~\cite{SudholtW19}. It seems to be an interesting direction for future research to find out to what extent such an approach can be used also for other lower bound analyses.

\newcommand{\etalchar}[1]{$^{#1}$}

%\bibliographystyle{alpha}
%%\bibliography{../allesea,../ich}
%\bibliography{../../IPL/eigene_arbeiten/ich_master,../../IPL/fremde_arbeiten/alles_ea_master}

%\end{large}
}%sloppy, please do not remove

\end{document}